\pdfoutput=1
\documentclass[10pt,twocolumn]{article}
\usepackage[utf8]{inputenc}
\usepackage[T1]{fontenc}
\usepackage{amsmath,amssymb,amsthm}
\usepackage{graphicx}
\usepackage{booktabs}
\usepackage[margin=0.75in]{geometry}
\usepackage{xcolor}
\usepackage[colorlinks=true,linkcolor=ink,citecolor=teal,urlcolor=teal]{hyperref}
\usepackage{algorithm}
\usepackage{algpseudocode}
\usepackage{caption}
\usepackage{subcaption}
\usepackage{enumitem}
\usepackage{titlesec}
\usepackage{textcomp}
\newcommand{\eur}{\texteuro}
\newcommand{\ypart}[2]{\par\addvspace{1.0\baselineskip}%
\begin{center}\rule{0.5\columnwidth}{0.5pt}\\[2.5pt]%
{\scshape\large Part #1}\,\textbf{\textbar}\,{\large\itshape #2}\\[2.5pt]%
\rule{0.5\columnwidth}{0.5pt}\end{center}\addvspace{0.4\baselineskip}\par\noindent}
\newtheorem{theorem}{Theorem}
\usepackage{tikz}
\usetikzlibrary{arrows.meta,positioning,shapes.geometric,calc,fit,backgrounds}

\definecolor{ink}{HTML}{13202B}
\definecolor{teal}{HTML}{2F9B94}
\definecolor{copper}{HTML}{C2703D}
\definecolor{amber}{HTML}{D08A3E}
\definecolor{paperbg}{HTML}{F4F1EA}

\tikzset{
  box/.style={rounded corners=2pt,draw=ink,line width=0.7pt,fill=white,
    align=center,font=\footnotesize,inner sep=4pt,minimum height=7mm},
  tealbox/.style={box,draw=teal,fill=teal!8},
  copbox/.style={box,draw=copper,fill=copper!8},
  inkbox/.style={box,draw=ink,fill=ink!6},
  flow/.style={-{Stealth[length=2mm]},line width=0.7pt,draw=ink},
  flowt/.style={-{Stealth[length=2mm]},line width=0.7pt,draw=teal},
  lbl/.style={font=\scriptsize\itshape,text=ink!70},
}

\titleformat{\section}{\normalfont\large\bfseries\color{ink}}{\thesection}{0.6em}{}
\titleformat{\subsection}{\normalfont\normalsize\bfseries\color{ink}}{\thesubsection}{0.5em}{}
\newtheorem{definition}{Definition}
\title{\vspace{-2.5em}\color{ink}\bfseries YUKTI: From Natural-Language Situations to\\
Robust, Verifiable Decisions\\[3pt]
{\large\mdseries An Uncertainty-Typed Proposition IR, Assumption-Robust Pareto\\
Frontiers, and a Regret Certificate --- why a language model should\\
\emph{formulate}, not \emph{solve}}}
\author{Suyash Mishra\\
\small AI Researcher\\
\small \texttt{\{zurich.suyash@gmail.com\}}}
\date{\small June 2026}

\begin{document}
\twocolumn[
\maketitle
\begin{abstract}
\noindent
Language models increasingly turn a worded situation into a numeric plan---and the
dominant pipelines (NL4Opt, OptiMUS, ORLM/LLMOPT, OR-LLM-Agent) commit to a
\emph{single objective} and \emph{point-valued} coefficients, then solve once. For
decisions that allocate real budget, real field effort, or real clinical attention,
that confidence \emph{is} the failure mode: every objectified number is an
assumption, and a plan that is optimal only if the guesses are exactly right is
silently fragile. We name the risk \emph{mimicry of computation}---output with the
\emph{form} of an optimized recommendation but none of its substance.
\par\smallskip
We present \textbf{YUKTI}, which changes the \emph{target} of autoformulation. Its
intermediate representation is a \emph{Typed Proposition} graph in which every
relationship carries a shape prior, a coefficient \emph{uncertainty distribution},
and provenance. On it, YUKTI (i) routes each stage to an exact, nonlinear, or
evolutionary solver by probing structure; (ii) couples stages through a
\emph{distributional} Pareto hand-off; (iii) introduces \textbf{Assumption-Robust
Pareto Frontiers (ARPF)}, resampling the assumptions---including structural,
$\varepsilon$-contamination misspecification---to score how often each action
survives ($\rho$); and (iv) when no data exist, synthesizes a benchmark-faithful
foundation (\textbf{SRJANA}). We \emph{prove} a regret bound making $\rho$ an exact
factor of decision regret, and we surface \emph{decision traceability}---which
segments compose an action and which constraint binds---so a recommendation is
auditable, not merely persuasive.
\par\smallskip
We validate the central claim three ways. Under controlled structural
misspecification, the robust compromise cuts mean and tail regret by over $90\%$
versus a naive point plan. On a regulated commercial decision (an oncology brand during
loss of exclusivity), we optimize inside a \emph{lawful}---MLR/consent/frequency/KAM
---action space against a \emph{prescribing-proxy} objective and price the downside
in euros. Finally, on a \emph{real} public dataset of $41{,}188$ marketing decisions
we did not generate, an out-of-sample backtest shows the robust rule beating the
logged status quo by $34\%$ and the naive point rule by $4\%$ while measurably
reducing the optimizer's curse. The solvers are standard and we make no
benchmark-SOTA claim; the contribution is the uncertainty-typed IR, ARPF and its
regret bound, the distributional multi-stage hand-off, decision-traceability
reporting, and a deployment posture as a decision \emph{stress-testing} layer---not
a system of record. A head-to-head locates the limit of LLM reasoning precisely: an
LRM handed the correct numbers, and classical single-objective optimization, both
incur ${\sim}47\times$ the held-out regret of YUKTI's robust compromise---an LRM is a
\emph{formulator}, not a solver. Finally we show where the formalism itself must
change: under long-range causal coupling between actions, the forward multi-stage
hand-off becomes unsound---its regret grows with coupling and its per-stage
certificate turns optimistic---locating the boundary at which a hand-off must become
a backward-induction causal policy.
\end{abstract}
\vspace{1em}
]

\ypart{I}{The Problem}
\section{Introduction}
A recurring request from decision-makers is deceptively simple: ``here is my
situation in words---tell me what to do.'' Underneath it sits an optimization
problem that nobody has written down. The decision-maker knows the
\emph{levers} they control, the \emph{signals} they care about, and the
\emph{limits} they must respect, but not the variables, objective, and
constraints of a formal program. Recent work has shown that language models can
bridge this gap: given a description, they can extract variables, constraints,
and an objective and emit solver code~\cite{nl4opt,optimus,orlm}.

Yet the prevailing formulation is narrow in three ways that matter precisely when
the stakes are highest. First, it is \emph{single-objective}: real briefs trade
reach against cost against risk, and collapsing them into one weighted scalar
hides the trade-off the decision-maker actually wants to see. Second, it is
\emph{single-shot}: many decisions are \emph{sequential}---you first design a
program, then decide how aggressively to roll it out, and the second decision is
constrained by what the first achieved. Third, and most important, it is
\emph{point-valued}: to turn ``oncologists respond well to clinical content'' into
a number, the model must invent a coefficient. That coefficient is an assumption,
and a plan that is optimal only at the assumed value can be fragile.

\begin{figure*}[t]
\centering
\begin{tikzpicture}[node distance=4mm]
\node[inkbox,text width=1.9cm] (brief) {Worded\\decision brief};
\node[copbox,text width=2.0cm,right=9mm of brief,yshift=10mm] (a1) {LLM\\autoformulation};
\node[copbox,text width=2.2cm,right=4mm of a1] (a2) {Point-valued model\\(one objective)};
\node[copbox,text width=1.8cm,right=4mm of a2] (a3) {Single\\optimum};
\node[copbox,text width=2.6cm,right=4mm of a3] (a4) {\textbf{Brittle} if an assumed\\number is wrong};
\node[tealbox,text width=2.2cm,right=9mm of brief,yshift=-10mm] (b1) {Typed--uncertain IR\\+ web anchors (SRJANA)};
\node[tealbox,text width=1.9cm,right=4mm of b1] (b2) {Structure-aware router};
\node[tealbox,text width=2.2cm,right=4mm of b2] (b3) {Pareto frontier\\(exact/NLP/EA)};
\node[tealbox,text width=2.6cm,right=4mm of b3] (b4) {\textbf{Robust} compromise\\with score $\rho$};
\draw[flow] (brief) to[out=40,in=180] (a1.west); \draw[flow](a1)--(a2);\draw[flow](a2)--(a3);\draw[flow](a3)--(a4);
\draw[flowt] (brief) to[out=-40,in=180] (b1.west); \draw[flowt](b1)--(b2);\draw[flowt](b2)--(b3);\draw[flowt](b3)--(b4);
\node[lbl,above=0.5mm of a1] {prevailing pipeline};
\node[lbl,below=0.5mm of b1] {YUKTI (this work)};
\end{tikzpicture}
\caption{\textbf{Why this is critical as LLMs enter every decision.} The same
worded brief can be objectified two ways. The prevailing pipeline (top) commits to
a single objective and point-valued coefficients, yielding one optimum that is
silently fragile when an invented number is off. YUKTI (bottom) keeps objectives
in conflict, carries each coefficient's uncertainty, routes the stage to the right
solver, and returns a compromise annotated with the probability $\rho$ that it
survives its own assumptions. When autoformulation drives clinical, public-health,
energy, and financial resource allocation at scale, that robustness signal is the
safeguard.}
\label{fig:criticality}
\end{figure*}
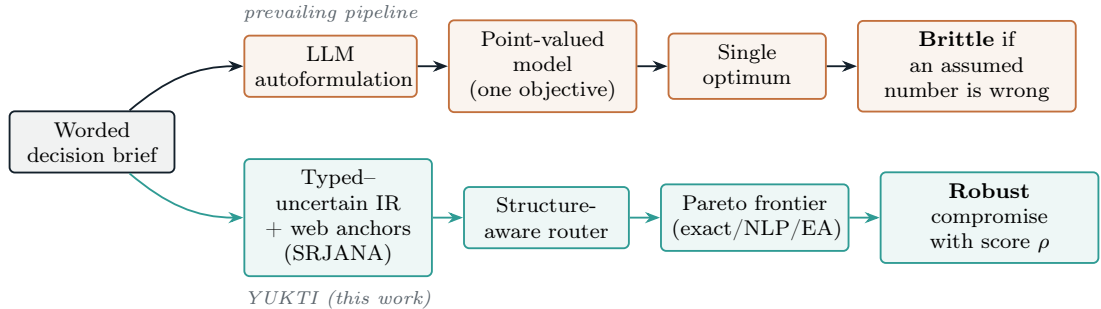

\textbf{Why this matters now.} As LLMs move from drafting text to \emph{making and
shaping decisions}---triaging patients, allocating screening outreach, sizing
energy subsidies, setting commercial strategy---the step that converts a
qualitative brief into numbers is being executed at population scale, often
invisibly. Each objectified coefficient is an assumption; a single-objective,
point-valued plan inherits all of them and emits no signal of its own fragility
(Fig.~\ref{fig:criticality}, top). When such plans drive real resource
allocation, silent brittleness is the failure mode. The discipline YUKTI
adds---keep conflicting objectives explicit, carry the uncertainty of every
invented number, and report how often a recommendation survives its own
assumptions---is exactly the safeguard that becomes critical once autoformulation
is embedded everywhere (Fig.~\ref{fig:criticality}, bottom).

\textbf{Two kinds of traceability.} The lesson is not that language models cannot
build such a model---a capable model can, and that is how the system in this paper
was built. The lesson is about the \emph{epistemic status of the output}. The risk
is \emph{mimicry of computation}: prose with the \emph{form} of an optimized
recommendation but none of its substance---no solved model, no stated coefficients,
no sensitivity, no robustness. Retrieval-augmented generation supplies \emph{source
traceability} (where a fact came from); high-stakes decisions additionally need
\emph{decision traceability}---the ability to trace how the \emph{recommended
action} depends on its inputs: which segments it comprises, which assumptions move
it, and which constraint binds. A recommendation can be perfectly sourced and still
be operationally unreliable. YUKTI targets decision traceability directly, and
\S\ref{sec:trace} adds two mechanisms---segment attribution and shadow prices---so
the property is delivered as output, not asserted in narrative.

\textbf{This paper.} We propose \textbf{YUKTI}\footnote{Sanskrit \emph{yukti}:
the fitting means, a reasoned device or stratagem.}, a framework that keeps the
LLM where it is strong---reading a situation and naming its structure---but
changes the \emph{target} of formulation and the \emph{machinery} behind it:
\begin{itemize}[leftmargin=1.1em,itemsep=1pt,topsep=2pt]
\item a \textbf{Typed Proposition IR} in which each quantitative relationship is a
function $\varphi(x,u;\theta)$ carrying qualitative shape priors, a
\emph{distribution} over its coefficients $\theta$, and a provenance tag
(given / assumed / benchmark);
\item a \textbf{structure-aware router} that probes each stage numerically and
sends it to an exact, nonlinear, or evolutionary multi-objective solver;
\item a \textbf{multi-stage frontier hand-off} that chains heterogeneous stages
across a decision DAG; and
\item \textbf{Assumption-Robust Pareto Frontiers (ARPF)}: by resampling the IR's
coefficient distributions and re-ranking the discovered solutions, YUKTI reports,
for every candidate, the probability it remains feasible and non-dominated, and
selects a \emph{robust} compromise.
\item \textbf{Benchmark-Anchored Contextual Synthesis (SRJANA)}: a front-end that,
when no dataset exists, synthesizes a contextual dataset from the extracted spec
and web-sourced benchmark anchors, fits the propositions to it, and thereby makes
the system answer briefs \emph{dynamically}.
\item \textbf{Decision traceability}: two reporting mechanisms that make a
recommended action auditable---\emph{segment attribution} (which segments make up
the action and their priority) and \emph{shadow-price analysis} (which constraint
binds and the marginal value of relaxing it).
\item \textbf{A regret bound} (Thm.~\ref{thm:arpf}) that makes $\rho$ an exact
factor of pool-regret and holds under structural ($\varepsilon$-contamination)
misspecification, plus an \textbf{out-of-sample validation} on a held-out,
RCT-anchored world where the robust compromise cuts regret by ${>}90\%$ versus a
naive point plan.
\end{itemize}
We deliberately separate novelty from reuse (\S\ref{sec:positioning}). The
solvers are well-established; the contribution is the uncertainty-typed IR, the
ARPF mechanism in the autoformulation setting, and the heterogeneous hand-off. We
make no claim of benchmark superiority; this is a system-and-method proposal with
a verified end-to-end demonstration.

\textbf{Roadmap.} The paper is in four parts. \emph{Part~I} (this section,
\S\ref{sec:related}) frames the problem and the prior autoformulation line.
\emph{Part~II} develops the method: the typed-uncertain IR, the structure-aware
router, the multi-stage hand-off, ARPF and its regret bound, SRJANA, and the
implementation. \emph{Part~III} is evidence: a worked example, a three-domain
demonstration, system evaluation, decision traceability, validation under
misspecification, validation on \emph{real} data, and a regulated deployment.
\emph{Part~IV} maps the boundaries: a head-to-head against LLM reasoning and
single-objective optimization, the point at which long-range causal coupling forces
the formalism to change, and our positioning, limitations, and future work.

\section{Related Work}
\label{sec:related}
\textbf{LLMs for optimization modeling.} NL4Opt framed natural-language-to-model
as extraction of variables, constraints, and objectives~\cite{nl4opt}.
Chain-of-Experts introduced multi-agent reasoning for operations
research~\cite{coe}. OptiMUS and OptiMUS-0.3 made the loop concrete---formulate,
generate solver code, execute, debug~\cite{optimus,optimus03}. ORLM fine-tuned
open models on synthetic instances and released the IndustryOR
benchmark~\cite{orlm}; LLMOPT codified a five-element structure (sets, parameters,
variables, objectives, constraints)~\cite{llmopt}; OR-LLM-Agent added explicit
task decomposition with a reasoning model~\cite{orllmagent}. The newest systems
emphasize a \emph{persisted} intermediate representation and post-hoc analysis:
ORPilot stores a validated JSON IR and supports what-if queries~\cite{orpilot};
OptiRepair and OptiLoop close the loop with solver-verified diagnosis and repair,
arguing for ``operational rationality'' beyond mere
feasibility~\cite{optirepair,optiloop}. All of these treat coefficients as fixed
once elicited, and almost all target single-objective formulation accuracy.
YUKTI inherits the IR-and-solver discipline but makes the IR \emph{uncertain and
typed}, and makes the target \emph{multi-objective and multi-stage}.

\textbf{Exact multi-objective programming.} The $\varepsilon$-constraint method
generates Pareto-optimal points by optimizing one objective while bounding the
rest~\cite{mavrotas2009}. AUGMECON2 augments it with lexicographic slack terms to
avoid weakly-efficient points and skip redundant iterations, producing the exact
Pareto set for integer programs~\cite{augmecon2}; AUGMECON-R extends practicality
to many objectives~\cite{augmeconr}. We use an augmented $\varepsilon$-constraint
generator on HiGHS for affine stages.

\textbf{Evolutionary multi/many-objective optimization.} NSGA-II combined fast
non-dominated sorting with crowding distance~\cite{nsga2}; NSGA-III replaced
crowding with reference-direction association for many
objectives~\cite{nsga3}, with directions placed by the Das--Dennis
simplex construction~\cite{dasdennis}; MOEA/D decomposes the problem into scalar
subproblems~\cite{moead}. We access these through pymoo~\cite{pymoo}, along with
its compromise-programming and high-trade-off MCDM utilities.

\textbf{Sequential decisions under uncertainty.} Multistage stochastic
programming models here-and-now versus recourse decisions over a scenario tree
with nonanticipativity, solved by nested Benders / L-shaped
decomposition~\cite{birge1985,birgebook}; the equivalent deterministic program
characterizes its structure~\cite{wets1974,shapiro}. YUKTI's multi-stage layer is
a deterministic \emph{frontier hand-off}---a hierarchical coupling of stages---
rather than full recourse; we treat coefficient uncertainty through ARPF instead
of a scenario tree, and discuss the relationship in \S\ref{sec:limitations}.

\textbf{Frontier uncertainty.} In multi-objective Bayesian optimization, qEHVI
and qNEHVI optimize expected hypervolume improvement and, crucially, integrate
over uncertainty \emph{in the Pareto frontier} induced by noisy
observations~\cite{qehvi,qnehvi}. ARPF is the same idea displaced to a different
source: it integrates over uncertainty in the frontier induced by uncertain
\emph{model coefficients} elicited by an LLM, rather than noisy measurements.
Robust optimization more broadly seeks solutions stable across an uncertainty
set~\cite{robust,bertsimas}; ARPF is a sampling-based, frontier-level instance
tailored to autoformulated propositions.

\ypart{II}{The YUKTI Method}
\section{The YUKTI Framework}
\label{sec:framework}
\begin{figure*}[t]
\centering
\includegraphics[width=0.92\textwidth]{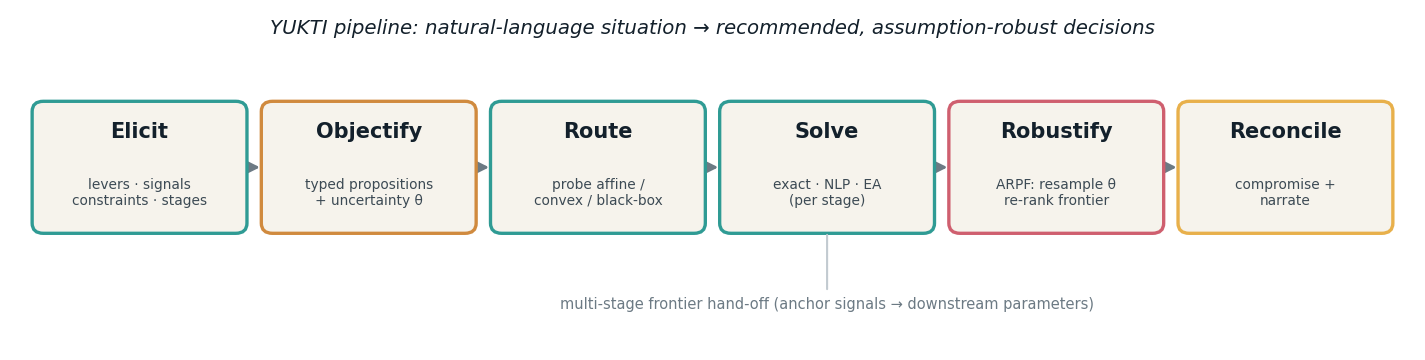}
\caption{The YUKTI pipeline. A worded situation is elicited into levers, signals,
constraints and a stage DAG; each relationship is \emph{objectified} into a typed
proposition with an uncertainty distribution; stages are routed to an exact, NLP,
or evolutionary solver; solved with a Pareto \emph{hand-off} across stages;
robustified by resampling assumptions (ARPF); and reconciled into a recommended,
robustness-scored compromise.}
\label{fig:pipeline}
\end{figure*}

A problem is a directed acyclic graph of \emph{stages}
$\mathcal{S}=\{s_1,\dots,s_T\}$ in topological order. Each stage $s$ owns levers
$x_s\in X_s$, a set of objective \emph{signals}, constraints, and a list of
upstream signal keys it consumes. Solving a stage yields a Pareto set; a
compromise is selected and its signal values are injected as parameters $u$ for
downstream stages (Fig.~\ref{fig:pipeline}). Formally, stage $s$ is the
multi-objective program
\begin{equation}
\begin{aligned}
\min_{x\in X_s}\ & F_s(x,u_s;\theta_s)=\big(f_{s,1},\dots,f_{s,m_s}\big),\\
&\text{s.t. } g_{s,j}(x,u_s;\theta_s)\le 0,
\end{aligned}
\label{eq:stage}
\end{equation}
where maximization signals are negated into $\min$ form, $u_s$ collects upstream
hand-offs, and $\theta_s$ are coefficients drawn from the proposition ledger.

\subsection{Pareto optimality}
For min-sense vectors, $a$ \emph{dominates} $b$ ($a\prec b$) iff
$a_i\le b_i\,\forall i$ and $a_k<b_k$ for some $k$. The Pareto set of a stage is
$\{x^\star: \nexists\, x\in X_s,\ F_s(x)\prec F_s(x^\star)\}$. YUKTI's solvers
return a finite non-dominated approximation $\mathcal{F}_s$.

\section{Typed Proposition Intermediate Representation}
\label{sec:ir}
The distinguishing object is the proposition.
\begin{definition}[Typed proposition]
A typed proposition is a tuple
$p=\langle \varphi,\ \theta,\ \Theta,\ \sigma,\ \pi\rangle$ where
$\varphi(x,u;\theta)\in\mathbb{R}$ is a parametric functional form;
$\theta$ are nominal coefficients; $\Theta=\{(\mathrm{dist}_k,\eta_k)\}$ assigns
each uncertain coefficient a multiplicative distribution with relative spread
$\eta_k$ (normal, lognormal, or uniform); $\sigma$ records qualitative shape
priors (monotonicity $\pm$ per lever; curvature: linear / convex / saturating /
bilinear); and $\pi\in\{\text{given, assumed, benchmark}\}$ is provenance.
\end{definition}
A signal is a proposition plus a sense (max/min); a constraint is a proposition
plus an operator and right-hand side. Two design choices matter. First, the same
proposition object can back both a signal and a constraint---e.g.\ a
\emph{fatigue} proposition is simultaneously an objective to minimize and a
guardrail $\le 0.35$---so that a single resampled draw moves both consistently
(shared identity, \S\ref{sec:arpf}). Second, because $\varphi$ is an arbitrary
callable, the IR spans linear, smooth-nonlinear, and nonconvex/black-box
relationships without changing the downstream machinery; structure is
\emph{discovered}, not declared (\S\ref{sec:router}).

A coefficient draw is multiplicative, e.g.\ for a lognormal coefficient
$\theta_k^{(j)}=\theta_k\cdot e^{\,\xi},\ \xi\sim\mathcal{N}(0,\eta_k^2)$,
keeping signs and order of magnitude while perturbing the elicited guess. This is
what turns ``the model invented a number'' from a liability into a quantified
input.

\section{Structure-Aware Routing}
\label{sec:router}
Before solving, YUKTI decides \emph{how}. For a stage with fixed $u$ and nominal
$\theta$, it forms the stacked map $h(x)=(F_s(x),G_s(x))$ on the continuous
relaxation and probes it on random pairs $x^{(1)},x^{(2)}$ with interpolation
weight $\alpha$ (Alg.~\ref{alg:route}). Affinity is accepted if
$h(\alpha x^{(1)}{+}(1{-}\alpha)x^{(2)})$ matches
$\alpha h(x^{(1)}){+}(1{-}\alpha)h(x^{(2)})$ within a scaled tolerance for all
pairs; convexity (in min-sense) is accepted if the interpolated point never lies
below the chord (Jensen). The route is then: \emph{affine} $\rightarrow$ exact
(MO-LP, or MO-MILP if integer levers are present); \emph{smooth, convex,
continuous} $\rightarrow$ NLP; otherwise $\rightarrow$ evolutionary. Probing the
\emph{relaxation} is essential: rounding integer levers turns an affine model
into a step function and would defeat the linearity test, so structure is
detected continuously while integrality is enforced inside the exact solver.

\begin{algorithm}[t]
\caption{Structure-aware routing of a stage}
\label{alg:route}
\begin{algorithmic}[1]
\Require stage $s$, upstream $u$, probes $R$, pairs $N$, tol $\tau$
\State sample $R$ points; set scale $\leftarrow$ mean $|h|$ per component
\State affine, convex $\leftarrow$ true, true
\For{$N$ random pairs $(x^{(1)},x^{(2)})$, $\alpha\in(0,1)$}
  \State $\Delta \leftarrow h(\alpha x^{(1)}{+}(1{-}\alpha)x^{(2)}) - [\alpha h(x^{(1)}){+}(1{-}\alpha)h(x^{(2)})]$
  \If{$\max|\Delta| > \tau\cdot(\text{scale}{+}1)$} affine $\leftarrow$ false \EndIf
  \If{any component of $\Delta > \tau\cdot(\text{scale}{+}1)$} convex $\leftarrow$ false \EndIf
\EndFor
\If{affine} \Return exact \Comment{MILP if integer levers else LP}
\ElsIf{convex and continuous} \Return nlp
\Else{} \Return ea
\EndIf
\end{algorithmic}
\end{algorithm}

\section{Solvers}
\label{sec:solvers}
\textbf{Exact (affine).} When a stage is affine, YUKTI extracts the coefficient
matrices by finite differences on the relaxation (the gradient is constant), then
runs an augmented $\varepsilon$-constraint generator~\cite{augmecon2}. With the
first objective as primary and the rest bounded on a grid between their payoff
ranges $[\,f_k^{\text{ideal}},f_k^{\text{nadir}}\,]$,
\begin{equation}
\min_{x\in X_s,\,s_k\ge 0}\; f_1(x)-\delta\!\!\sum_{k\ge 2}\frac{s_k}{r_k}
\;\;\text{s.t. } f_k(x)+s_k=e_k,
\label{eq:augmecon}
\end{equation}
where $r_k$ is the range of objective $k$ and $\delta$ a small augmentation. Each
grid point is an LP (HiGHS \texttt{linprog}) or MILP (HiGHS \texttt{milp}) with
integrality enforced on $x$; results are filtered to the non-dominated set. For
two or three objectives this yields the exact (or near-exact) Pareto front.

\textbf{Nonlinear (convex, smooth).} YUKTI applies the $\varepsilon$-constraint
with SLSQP, sweeping the secondary-objective grid with multiple random restarts,
then filters non-dominance.

\textbf{Evolutionary (nonconvex / black-box / mixed-integer).} Otherwise YUKTI
uses NSGA-II for $\le 2$ objectives and NSGA-III with Das--Dennis reference
directions for $\ge 3$~\cite{nsga2,nsga3,dasdennis}, via pymoo~\cite{pymoo}, with
constraint handling by feasibility-first domination and integer levers realized
by rounding in evaluation.

\section{Multi-Stage Frontier Hand-off}
\label{sec:multistage}
Stages are solved in topological order. After stage $s$ produces $\mathcal{F}_s$,
YUKTI selects a (robust) compromise $x^\star_s$ and \emph{hands off} its signal
values:
\begin{equation}
u_{s+1} \supseteq \big\{\, \text{key}(f_{s,i}) \mapsto f_{s,i}(x^\star_s) \,\big\}_{i=1}^{m_s}.
\end{equation}
Downstream propositions read these as parameters. In the demonstration, the
email-engagement stage exposes \emph{open rate} and \emph{CTR}; the transition
stage consumes them through an engagement-quality term
$\mathrm{EQ}=\min(1,\,0.6\cdot\text{open}+4\cdot\text{CTR})$ that scales how much
share-of-voice and continuity the rollout can preserve. This is a hierarchical
coupling (strategic stage fixes the environment of the operational stage), not a
recourse model; \S\ref{sec:limitations} states the trade-off.

\textbf{Distributional hand-off.} Passing the point values $f_{s,i}(x^\star_s)$
commits the downstream stage to a single number---in tension with the paper's own
thesis against point values. But under ARPF we already hold, at $x^\star_s$, the
\emph{distribution} of each handed-off signal across the $K$ assumption draws.
YUKTI therefore hands off that distribution (its quantiles), and stage $s{+}1$
evaluates feasibility and value in expectation over it, so a stage-1 signal that is
high in the mean but heavy-tailed no longer masks downstream risk.
\S\ref{sec:validation} exhibits a point hand-off that reports a stage-2 floor as
satisfied while the distributional hand-off prices a $60\%$ violation probability.

\section{Assumption-Robust Pareto Frontiers}
\label{sec:arpf}
A frontier computed at the nominal $\theta$ answers ``what is optimal if my
elicited numbers are exactly right?''---rarely the operative question. ARPF
answers ``what stays good when they are wrong?'' Let $\mathcal{X}_s$ be the
discovered decision pool. For $K$ draws $\theta^{(k)}\sim\Theta$ (shared across a
proposition's signal/constraint roles via object identity), recompute objectives
and feasibility for every pooled decision, take the non-dominated subset of the
feasible ones, and accumulate membership:
\begin{equation}
\rho(x)=\frac{1}{K}\sum_{k=1}^{K}\mathbf{1}\!\left[x \text{ feasible} \wedge x\in
\mathrm{ND}\big(F_s(\cdot;\theta^{(k)})\big)\right].
\label{eq:arpf}
\end{equation}
$\rho(x)$ is the empirical probability that $x$ survives as a sensible choice
under assumption uncertainty; the per-decision objective spread
(mean $\pm$ std across draws) traces an \emph{envelope} around the nominal
frontier (Fig.~\ref{fig:onco}b). We then choose the compromise \emph{among robust
solutions}: restrict to $\rho(x)\ge$ the median, and within that set minimize the
normalized distance to the ideal point (compromise programming),
\begin{equation}
x^\star_s=\arg\min_{x:\rho(x)\ge \rho_{0.5}}\;\sum_i w_i
\Big(\tfrac{f_{s,i}(x)-f_{s,i}^{\text{ideal}}}{f_{s,i}^{\text{nadir}}-f_{s,i}^{\text{ideal}}}\Big)^2 .
\label{eq:compromise}
\end{equation}
Finally, a one-at-a-time tornado ranks which coefficients move the recommendation
most (Figs.~\ref{fig:portfolio}b,~\ref{fig:onco}c), separating \emph{assumed}
from \emph{benchmark} provenance so attention goes where it is both influential
and weakly grounded.

\textbf{Structural misspecification ($\varepsilon$-contamination).} Resampling
$\theta$ within a parametric family addresses \emph{parameter} error but not a
\emph{wrong functional form}---the error that usually dominates in practice. YUKTI
therefore lets $\Theta$ be an $\varepsilon$-contamination class
$\Theta=(1-\varepsilon)\Theta_0+\varepsilon\,Q$: with probability $\varepsilon$ a
draw replaces a proposition's form with an alternative from $Q$ (for instance
injecting a diminishing-returns interaction absent from a nominal additive form).
ARPF over this $\Theta$ scores robustness to the model being \emph{structurally}
wrong, not merely mis-calibrated; \S\ref{sec:validation} shows this is precisely
what lets the robust choice generalize to a held-out world whose structure the
nominal model never represented.

\begin{algorithm}[t]
\caption{YUKTI multi-stage solve with ARPF}
\label{alg:engine}
\begin{algorithmic}[1]
\Require problem $\mathcal{S}$ in topological order; draws $K$
\State $u\leftarrow\{\}$
\For{stage $s$ in order}
  \State method $\leftarrow$ \textsc{Route}$(s,u)$ \Comment{Alg.~\ref{alg:route}}
  \State $\mathcal{F}_s\leftarrow$ \textsc{Solve}$_{\text{method}}(s,u)$
  \State $\{\rho,\text{env}\}\leftarrow$ \textsc{ARPF}$(s,u,\mathcal{F}_s,K)$ \Comment{Eq.~\ref{eq:arpf}}
  \State $x^\star_s\leftarrow$ \textsc{RobustCompromise}$(\mathcal{F}_s,\rho)$ \Comment{Eq.~\ref{eq:compromise}}
  \State $u\leftarrow u\cup\{\,f_{s,i}(x^\star_s)\,\}_i$ \Comment{hand-off}
\EndFor
\State \Return $\{(\mathcal{F}_s,x^\star_s,\rho_s)\}$
\end{algorithmic}
\end{algorithm}

\section{A Regret Bound for ARPF}
\label{sec:theory}
ARPF's robustness score is not merely heuristic: it is an exact factor of decision
regret. Fix a finite pool $P$ and a distribution $\Theta$ over models $\theta$
(parameters \emph{and} functional form). Let $V(x;\theta)\in[0,V_{\max}]$ be a
bounded decision value with $V(x;\theta)=0$ when $x$ is infeasible under $\theta$,
let $V^\star(\theta)=\max_{x\in P}V(x;\theta)$ be the in-hindsight pool optimum, and
$\rho(x)=\Pr_\theta[x\in\arg\max_{x'\in P}V(x';\theta)]$ the decision-robustness of
Eq.~\ref{eq:arpf}. Define the \emph{pool-regret}
$R(\hat x)=\mathbb{E}_\theta[V^\star(\theta)-V(\hat x;\theta)]$.

\begin{theorem}[ARPF regret decomposition and certificate]
\label{thm:arpf}
For any $\hat x\in P$, writing
$\bar g(\hat x)=\mathbb{E}_\theta[V^\star-V(\hat x)\mid \hat x\notin\arg\max]$ and
$B(\hat x)=\operatorname*{ess\,sup}_\theta(V^\star(\theta)-V(\hat x;\theta))$:
\emph{(i)} the regret decomposes exactly,
$R(\hat x)=(1-\rho(\hat x))\,\bar g(\hat x)$;
\emph{(ii)} hence $R(\hat x)\le(1-\rho(\hat x))\,B(\hat x)$;
\emph{(iii)} for a shared bound $\bar B\ge B(\cdot)$, $\arg\max_x\rho(x)$ minimizes
the certificate $(1-\rho)\bar B$ and is the Bayes rule under $0$--$1$
pool-optimality loss; and
\emph{(iv)} $\Pr(|\hat\rho_K-\rho|\ge t)\le 2e^{-2Kt^2}$ (Hoeffding).
The result holds for \emph{arbitrary} $\Theta$, in particular an
$\varepsilon$-contamination class $\Theta=(1-\varepsilon)\Theta_0+\varepsilon Q$;
hence the bound accounts for structural misspecification, not only parameter noise.
\end{theorem}

\begin{proof}
(i) $V^\star-V(\hat x)=0$ on the event $\{\hat x\in\arg\max\}$, so by the law of
total expectation $R(\hat x)=\Pr[\hat x\notin\arg\max]\cdot\mathbb{E}[V^\star-V(\hat x)\mid\hat x\notin\arg\max]=(1-\rho(\hat x))\bar g(\hat x)$.
(ii) $0\le V^\star-V(\hat x)\le B(\hat x)$ a.s.\ and $=0$ on the argmax event, so
$R\le(1-\rho)B$. (iii) For shared $\bar B$, $(1-\rho(x))\bar B$ decreases in
$\rho(x)$; and under $0$--$1$ loss $L(x;\theta)=\mathbf 1[x\notin\arg\max]$,
$\mathbb{E}_\theta L=1-\rho(x)$, minimized by $\arg\max_x\rho$. (iv)
$\mathbf 1[x\in\arg\max]$ is Bernoulli$(\rho)$; Hoeffding over $K$ i.i.d.\ draws
gives the bound. The argument uses only boundedness and $\theta\sim\Theta$, so
$\Theta$ may mix functional forms.
\end{proof}

\noindent\textbf{Scope.} $R$ is regret against the \emph{best-in-pool}, not an
unattainable global optimum---the right object because ARPF re-ranks a fixed
discovered pool. Part (iii) certifies the $0$--$1$ pool-regret; for \emph{value}
regret under misspecification the robust compromise additionally uses the downside
(lower tail) of the same per-action value distribution, validated next. The
certificate is checked numerically: the identity (i) reproduces $R$ exactly and
(ii) holds with $\rho{=}0.24,\ \bar g{=}0.21,\ B{=}0.45$ (Fig.~\ref{fig:validation}b).

\section{Dynamic Contextual Data Synthesis}
\label{sec:bacs}
Autoformulation assumes the numbers exist---that a coefficient can be elicited or
a table queried. Often neither holds: a fresh brief arrives with conflicting
objectives, named levers, and \emph{no} data. The point-valued IRs of prior work
have nothing to fit. YUKTI's typed propositions invite a different move: if every
relationship is a parametric form with shape priors and an anchor, we can
\emph{synthesize} a contextual dataset that honours those priors and matches
public benchmarks, then fit the propositions to it. We call this
\textbf{Benchmark-Anchored Contextual Synthesis (SRJANA)}.

A \emph{scenario spec} lists context dimensions (each a categorical with levels
and a latent affinity per level), levers with bounds, and signals---each annotated
with a sense, a kind (rate / score / cost), signed lever effects (the shape
priors), and a \emph{benchmark anchor} $\mu^\star$ (a target marginal, ideally
retrieved from a trusted source at query time). SRJANA draws a population of context
units with latent affinity $z_u=\sum_{\text{dim}}\tilde a(\text{level}_u)$ and a
set of decision configurations $\{x_c\}$, then forms the linear predictor
$L_{c,u}=\sum_\ell \beta_{s,\ell}\tilde x_{c,\ell}+\gamma_s z_u$ for each signal
$s$. The defining step is \emph{calibration to the benchmark}: for a rate signal
choose the intercept $\alpha_s$ so the simulated grand mean matches the anchor,
\begin{equation}
\alpha_s:\quad \tfrac{1}{NM}\textstyle\sum_{c,u}\sigma\!\big(\alpha_s+L_{c,u}\big)=\mu^\star_s,
\label{eq:bacs}
\end{equation}
solved by one-dimensional root-finding ($\sigma$ monotone in $\alpha$); for a
linear score/cost signal $\alpha_s=\mu^\star_s-\overline{L}$ in closed form. SRJANA
then simulates outcomes, aggregates to configuration level, and fits the response
surfaces by OLS---returning coefficients \emph{and} standard errors that become
the propositions' nominal $\theta$ and uncertainty $\Theta$ (Alg.~\ref{alg:bacs}).
The synthetic data is therefore realistic by construction (its marginals are the
benchmarks) while its conditional structure obeys the elicited signs and the
optimizer inherits honest, data-derived uncertainty for ARPF.

Three additions make SRJANA faithful rather than merely calibrated. First,
benchmarks usually fix not just a mean but \emph{contrasts}---``rep-triggered
email earns roughly twice the click rate of mass.'' We therefore calibrate a
small parameter vector $(\alpha_s,\kappa_s)$ (intercept and a global effect scale)
against a system of moment equations,
\begin{equation}
\bar p(\alpha,\kappa)=\mu^\star,\qquad
\frac{\bar p_{\ell=1}(\alpha,\kappa)}{\bar p_{\ell=0}(\alpha,\kappa)}=r^\star_\ell,
\label{eq:srjana-moments}
\end{equation}
solved by least squares---so both the level and the key effect ratios match the
literature. Second, SRJANA injects the \emph{pathologies} real CRM extracts carry
(Apple Mail Privacy Protection open inflation, bounce/spam statuses, consent
gating), so downstream analytics confront the same artefacts they will in
production and (e.g.) correctly prefer click-to-open over raw open rate. Third,
and most important, SRJANA closes a \emph{self-consistency} loop: after synthesis
it refits the response surface and checks that the recovered coefficients match
the seeded priors. Across the three scenarios below, sign recovery is $100\%$ and
the calibrated marginals reproduce their targets to rounding
(Fig.~\ref{fig:srjana}). A dataset that could not recover its own generating
priors would be unfit to optimize against; SRJANA refuses to emit one.

\begin{algorithm}[t]
\caption{SRJANA: benchmark-anchored contextual synthesis}
\label{alg:bacs}
\begin{algorithmic}[1]
\Require spec (context, levers, signals w/ anchors $\mu^\star$, priors $\beta,\gamma$)
\State sample $N$ units; $z_u\gets$ standardized sum of level affinities
\State sample $M$ configs $x_c$; normalize to $\tilde x_c\in[0,1]^L$
\For{each signal $s$}
  \State $L_{c,u}\gets \sum_\ell \beta_{s,\ell}\tilde x_{c,\ell}+\gamma_s z_u$
  \State calibrate $(\alpha_s,\kappa_s)$ by Eq.~\ref{eq:srjana-moments} (rate) or closed form (score/cost)
  \State simulate outcomes; inject pathologies; aggregate to config means $\bar y_{s,c}$
  \State fit OLS $\bar y_{s}\!\sim\! x$; emit proposition $(\theta=\hat\beta,\ \Theta\propto \mathrm{SE})$
\EndFor
\State refit \& check sign/marginal recovery; resample if below tolerance
\State \Return dataset, fitted propositions \Comment{feed Alg.~\ref{alg:engine}}
\end{algorithmic}
\end{algorithm}

Together SRJANA and the elicitor make the system \emph{dynamic}: a natural-language
brief is parsed into a spec, anchors are pulled from the web, SRJANA manufactures a
calibrated data foundation, and the engine returns an assumption-robust
recommendation---no pre-existing dataset required.

\section{Implementation}
YUKTI is a Python package (\texttt{ir, router, solvers, mcdm, robust, engine,
cases, viz}). Exact solves use SciPy's HiGHS interface
(\texttt{linprog}/\texttt{milp})~\cite{scipy,highs}; nonlinear solves use SciPy
SLSQP; evolutionary solves use pymoo 0.6~\cite{pymoo}. The router, solvers, MCDM
and ARPF are solver-agnostic: a stage only needs to evaluate its propositions.
The same code path handles an exact mixed-integer program and a nonconvex
many-objective program, differing only in the route taken.

\ypart{III}{Evidence}
\section{Worked Example: One Brief, End to End}
\label{sec:demo}
We report two cases run end-to-end. These are illustrative: coefficients are
elicited placeholders, not measured data (\S\ref{sec:limitations}).

\subsection{Exact path: mixed-integer portfolio}
A campaign-mix stage chooses integer campaign counts across four channels to
maximize reach and minimize cost and risk, under a budget and an activity cap.
The router detects an affine model with integer variables and selects the exact
MO-MILP path; augmented $\varepsilon$-constraint on HiGHS returns \textbf{23
non-dominated points} (Fig.~\ref{fig:portfolio}a). The robust compromise commits
to a low-cost, low-risk mix (reach $34$, cost $25$, risk $0.9$) with
$\rho=0.99$; the tornado shows the recommendation is most exposed to the
reach-per-campaign coefficient of the highest-volume channel
(Fig.~\ref{fig:portfolio}b), flagged \emph{benchmark} rather than \emph{assumed}.

\begin{figure}[t]
\centering
\begin{subfigure}{0.49\columnwidth}
\includegraphics[width=\textwidth]{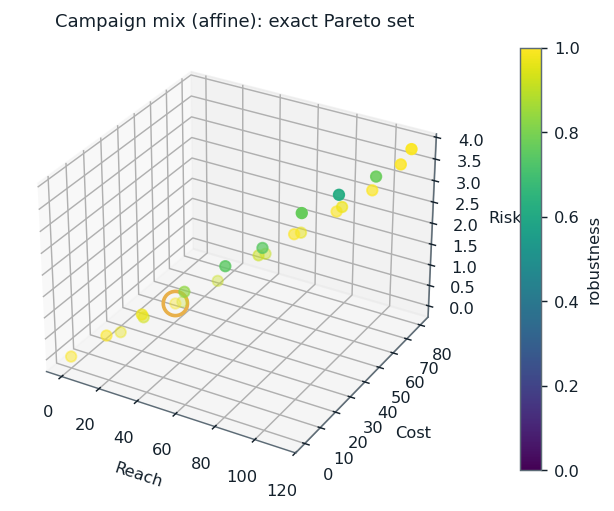}
\caption{exact Pareto set}
\end{subfigure}\hfill
\begin{subfigure}{0.49\columnwidth}
\includegraphics[width=\textwidth]{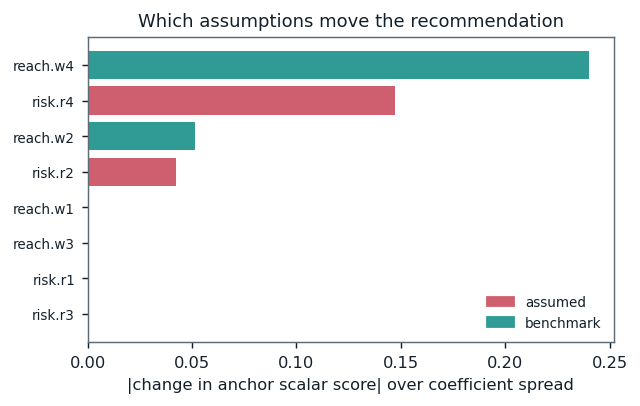}
\caption{assumption tornado}
\end{subfigure}
\caption{Mixed-integer portfolio solved exactly (augmented
$\varepsilon$-constraint, HiGHS): 23 Pareto points coloured by robustness, with
the robust compromise ringed (a); coefficients ranked by influence on the
recommendation (b).}
\label{fig:portfolio}
\end{figure}

\subsection{Multi-stage nonconvex: oncology channel shift}
The second case is a worded pharmaceutical brief: during loss of exclusivity,
pivot two lung-cancer brands from high-frequency field visits to scalable email,
without losing share of voice. YUKTI formulates it as two coupled stages.
\emph{Stage 1 (email engagement design)} has five levers---personalization,
rep-triggered share, clinical-content share, send-time optimization, cadence---and
four signals (open rate, CTR, cost, fatigue risk) under a budget and a fatigue
guardrail. The bilinear CTR term and the kinked fatigue penalty make the stage
nonconvex; the router selects NSGA-III, returning a \textbf{40-point} frontier
(Fig.~\ref{fig:onco}a). The robust compromise ($\rho=0.97$) leans on
personalization and send-time optimization at \emph{low} cadence---reaching
$\approx 43\%$ open and $14\%$ CTR at zero fatigue risk inside budget---because
the fatigue guardrail caps what additional volume can buy. ARPF's envelope
(Fig.~\ref{fig:onco}b) and tornado (Fig.~\ref{fig:onco}c) identify the open-rate
personalization elasticity and cost coefficients as the assumptions most worth
verifying.

\emph{Stage 2 (field-to-email transition)} consumes Stage 1's open rate and CTR
through the engagement-quality term and chooses an integer transition window, a
shift pace, and a residual high-touch share, trading share-of-voice retention
against transition cost and engagement continuity, with a floor
$\text{SoV}\ge 0.6$. It is again nonconvex (integer window, bilinear pace term);
NSGA-III returns a \textbf{28-point} frontier. The robust compromise
($\rho=0.98$) phases the shift gradually, keeps a small ($\approx 5\%$) residual
high-touch tier for top targets, and retains $\approx 0.89$ share of voice at
$0.79$ continuity (Fig.~\ref{fig:onco}d). Table~\ref{tab:anchors} summarizes both
stages.

\begin{figure}[t]
\centering
\begin{subfigure}{0.49\columnwidth}
\includegraphics[width=\textwidth]{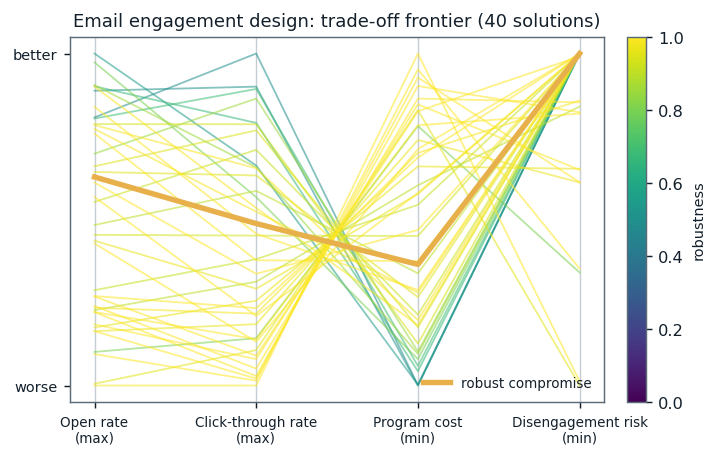}
\caption{stage 1 frontier}
\end{subfigure}\hfill
\begin{subfigure}{0.49\columnwidth}
\includegraphics[width=\textwidth]{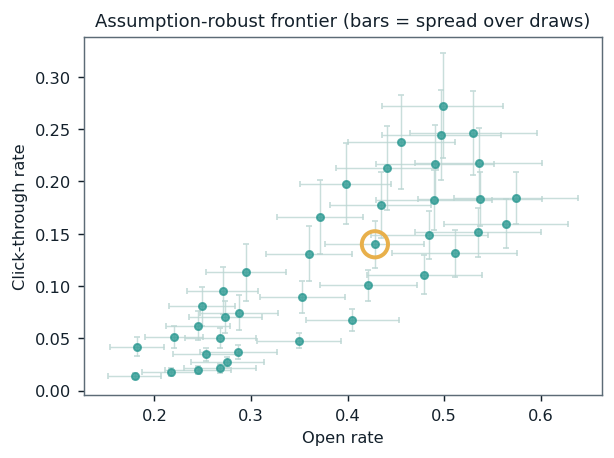}
\caption{ARPF envelope}
\end{subfigure}

\begin{subfigure}{0.49\columnwidth}
\includegraphics[width=\textwidth]{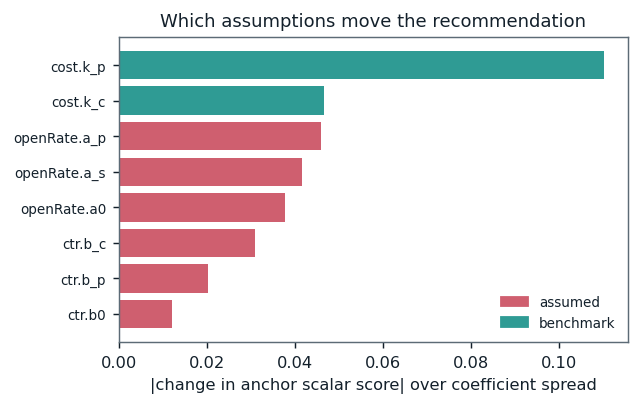}
\caption{stage 1 tornado}
\end{subfigure}\hfill
\begin{subfigure}{0.49\columnwidth}
\includegraphics[width=\textwidth]{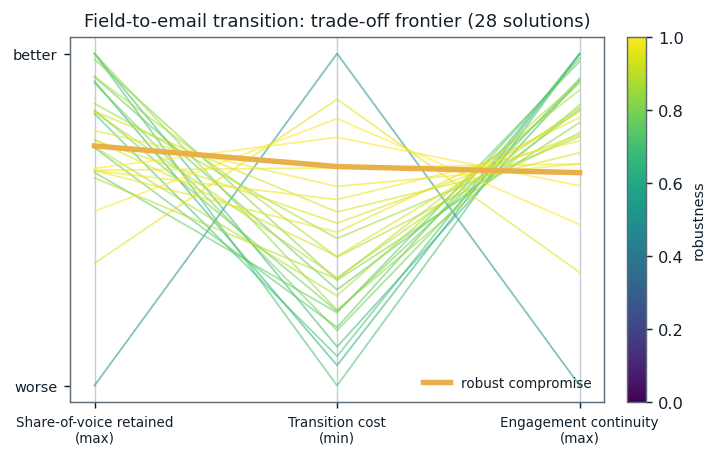}
\caption{stage 2 frontier}
\end{subfigure}
\caption{Two-stage oncology channel shift (NSGA-III + ARPF). Parallel-coordinate
frontiers are coloured by robustness with the robust compromise in gold (a,d);
(b) shows the assumption-robust envelope (objective spread over $K{=}200$ draws);
(c) ranks the assumptions that most move the stage-1 recommendation.}
\label{fig:onco}
\end{figure}

\begin{table}[t]
\centering\small
\caption{Robust compromises (verified run). Class is chosen by the router.}
\label{tab:anchors}
\setlength{\tabcolsep}{5pt}
\resizebox{\columnwidth}{!}{%
\begin{tabular}{@{}llrr@{}}
\toprule
Stage & Class & Front & $\rho$ \\
\midrule
Portfolio (campaign mix) & MO-MILP (exact) & 23 & 0.99 \\
Oncology s1 (engagement) & MO-EA (NSGA-III) & 40 & 0.97 \\
Oncology s2 (transition) & MO-EA (NSGA-III) & 28 & 0.98 \\
\bottomrule
\end{tabular}}
\end{table}

\section{Multi-Scenario Demonstration}
\label{sec:multiscenario}
To show the pipeline is domain-general, we drive three deliberately diverse briefs
through the \emph{same} SRJANA+YUKTI code with no per-domain tuning: oncology HCP
email engagement during loss of exclusivity (pharmaceutical commercial),
colorectal-cancer (CRC) screening uptake across deprivation strata (public
health), and residential heat-pump adoption across utility regions (energy).
Benchmark anchors are pulled from public sources at formulation time---IQVIA /
Veeva / phamax engagement rates for the first; CRC-screening randomized trials
(opt-out mailed FIT kits, serial SMS reminders) and national participation
statistics for the second; IRA/HEEHRA rebate levels and the heat-pump
\emph{additionality} literature for the third. In every case SRJANA reproduced the
benchmark marginals \emph{and} the key effect ratios to rounding
(Fig.~\ref{fig:srjana}), and the analytics recovered domain-correct drivers
(personalization and rep-triggering for email; mailed-FIT and patient navigation
for screening; rebate depth and installer coverage for heat pumps), i.e.\ $100\%$
seeded-sign recovery.

The router then sends each design stage to the solver its structure warrants. The
affine linear-probability response surfaces route to the exact
$\varepsilon$-constraint solver---the screening design yields a 135-point exact
frontier---while the heat-pump design, whose additionality couples to targeting,
is nonconvex and routes to NSGA-III; every second-stage rollout (a kinked,
bilinear continuity model) routes to the evolutionary solver (42, 33 and 42
points respectively). Table~\ref{tab:scenarios} reports the design-stage frontiers
and robust compromises; Fig.~\ref{fig:scenarios} shows them. One worded brief
$\rightarrow$ web-anchored synthesized data $\rightarrow$ \emph{routed}
multi-objective solve $\rightarrow$ assumption-robust compromise, across three
unrelated domains through one uniform path, is the generality claim made concrete.

\begin{table}[t]
\centering\small
\caption{Three diverse scenarios through one pipeline (verified run). SRJANA
reproduces every benchmark anchor and recovers $100\%$ of seeded driver signs;
the design stage routes to the solver its structure warrants.}
\label{tab:scenarios}
\setlength{\tabcolsep}{5pt}
\resizebox{\columnwidth}{!}{%
\begin{tabular}{@{}llccc@{}}
\toprule
Scenario & Domain & Obj. & Solver & Front / $\rho$\\
\midrule
Oncology email & Pharma commercial & 3 & MO-LP (exact) & 21 / 0.95\\
CRC screening & Public health & 3 & MO-LP (exact) & 135 / 1.00\\
Heat-pump adoption & Energy & 3 & MO-EA (NSGA-III) & 38 / 1.00\\
\bottomrule
\end{tabular}}
\end{table}

\begin{figure}[t]
\centering
\begin{subfigure}{0.49\columnwidth}\includegraphics[width=\textwidth]{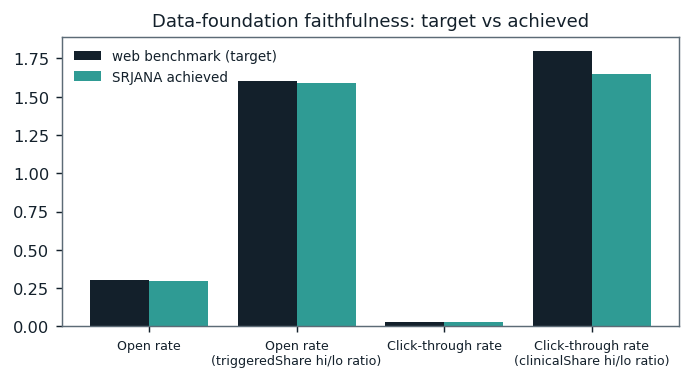}\caption{pharma}\end{subfigure}\hfill
\begin{subfigure}{0.49\columnwidth}\includegraphics[width=\textwidth]{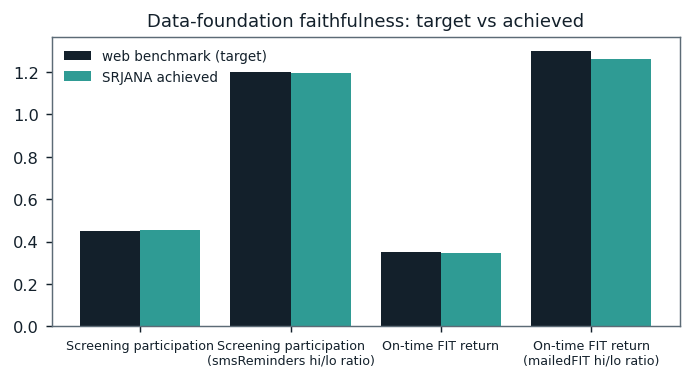}\caption{public health}\end{subfigure}
\\[2pt]
\begin{subfigure}{0.49\columnwidth}\includegraphics[width=\textwidth]{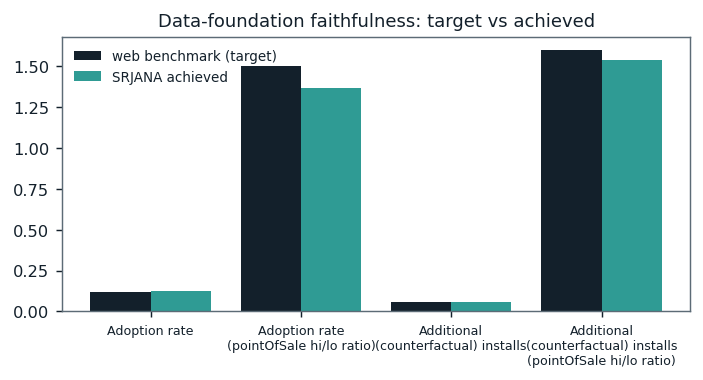}\caption{energy}\end{subfigure}
\caption{SRJANA self-consistency across three domains: the synthesized population
marginals \emph{and} the seeded effect ratios (teal) match the web-sourced
benchmark targets (ink) to rounding. A data foundation that cannot reproduce its
own anchors is rejected before the optimizer ever runs.}
\label{fig:srjana}
\end{figure}

\begin{figure*}[t]
\centering
\begin{subfigure}{0.32\textwidth}\includegraphics[width=\textwidth]{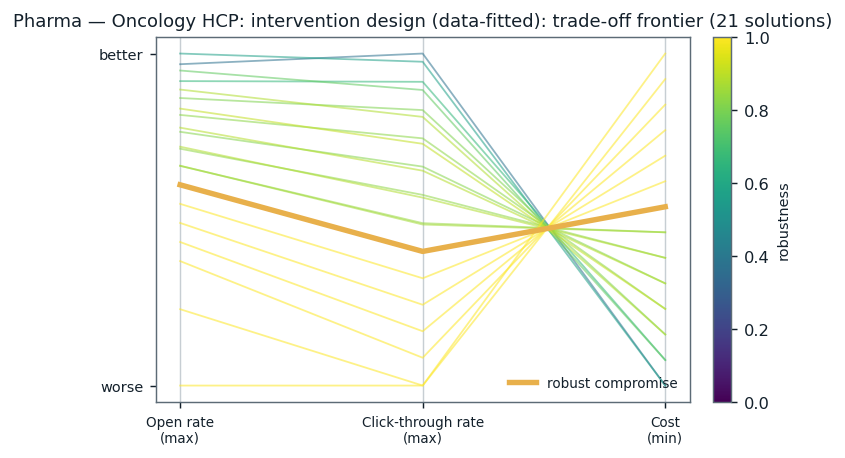}\caption{oncology email (exact MO-LP)}\end{subfigure}\hfill
\begin{subfigure}{0.32\textwidth}\includegraphics[width=\textwidth]{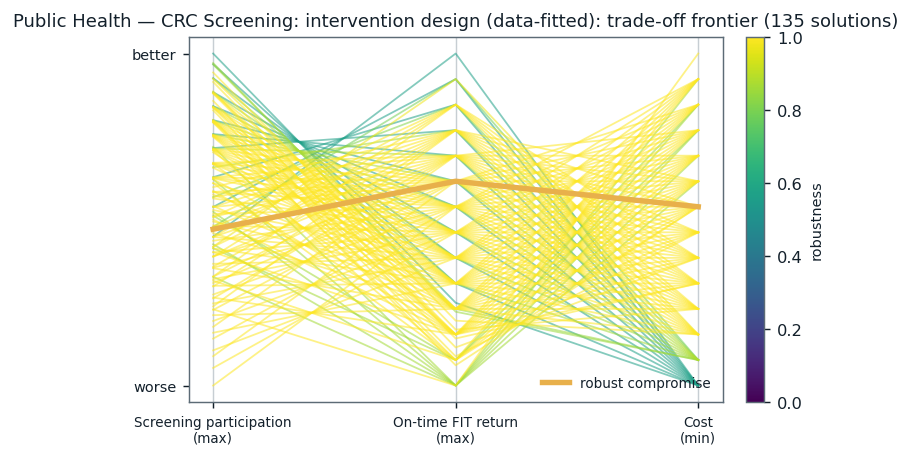}\caption{CRC screening (exact MO-LP)}\end{subfigure}\hfill
\begin{subfigure}{0.32\textwidth}\includegraphics[width=\textwidth]{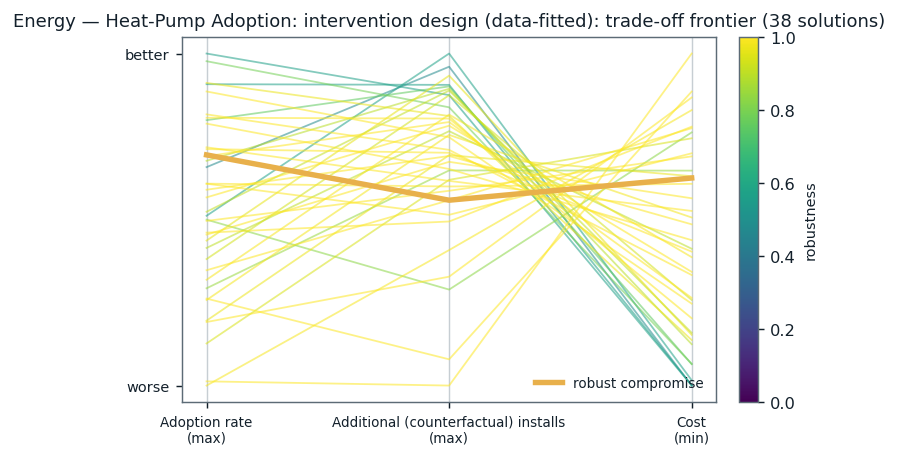}\caption{heat-pump (NSGA-III)}\end{subfigure}
\caption{Three diverse scenarios synthesized by SRJANA and solved by YUKTI.
Parallel-coordinate design-stage frontiers coloured by ARPF robustness; the robust
compromise is in gold. The pipeline is identical across domains; only the routed
solver and the web-anchored data differ.}
\label{fig:scenarios}
\end{figure*}

\section{Evaluation}
\label{sec:evaluation}
We evaluate the \emph{system}, not a benchmark leaderboard: for each scenario we
report (i) whether SRJANA's data foundation reproduces its web anchors
(faithfulness), (ii) the quality and spread of the discovered frontier
(normalized hypervolume and size), (iii) how sharply ARPF separates fragile from
robust solutions, (iv) the response-surface fit consumed by the optimizer, and
(v) wall-clock cost. Two quantities drive the evaluation. The ARPF robustness of
a candidate $x$ is the Monte-Carlo probability it survives its own assumptions,
\begin{equation}
\rho(x)=\frac1K\sum_{k=1}^{K}\mathbf 1\!\big[g(x;\theta^{(k)})\le 0\big]\,
\mathbf 1\!\big[x\in\mathrm{ND}(\mathcal P;\theta^{(k)})\big],
\label{eq:rho}
\end{equation}
where $\mathrm{ND}(\cdot)$ is the non-dominated set of the discovered pool
$\mathcal P$ re-ranked under resampled coefficients $\theta^{(k)}\sim\Theta$
(standard error
$\sqrt{\rho(1-\rho)/K}$). Frontier quality is the hypervolume of the min-sense
front $P$ after per-objective normalization to $[0,1]$,
\begin{equation}
\mathrm{HV}(P)=\Lambda\!\Big(\textstyle\bigcup_{p\in P}[\,p,\,r\,]\Big),\qquad r=1.1\cdot\mathbf 1 .
\label{eq:hv}
\end{equation}
Table~\ref{tab:eval} reports both across the three scenarios (verified run, $K{=}160$,
single core). SRJANA reproduces the benchmark \emph{marginals} to within $1\%$ in
every case; the larger residuals in the table are on the harder \emph{effect-ratio}
anchors (e.g.\ a clinical-content click lift), still within $\sim\!9\%$. The router
sends each design stage to the solver its structure warrants---two affine stages to
the exact $\varepsilon$-constraint solver (the screening design yields a 135-point
exact frontier) and the nonconvex heat-pump design to NSGA-III---and ARPF assigns a
wide robustness range in every case ($\rho$ as low as $0.37$ for aggressive
peripheral solutions), so the selected compromise ($\rho^\star$) is a genuine
choice, not an artifact of a flat frontier. End-to-end cost is $\approx\!14$\,s per
scenario including synthesis, analytics, both solves, and ARPF.

\begin{table*}[t]
\centering\small
\caption{System evaluation across three diverse scenarios (one pipeline, verified
run). HV is normalized hypervolume (Eq.~\ref{eq:hv}); $\rho^\star$ is the robustness
of the selected compromise and the range spans the whole frontier (Eq.~\ref{eq:rho});
faithfulness is the max relative error over benchmark anchors (marginals are
$<1\%$); $R^2$ is the primary response-surface fit; time is single-core wall clock.}
\label{tab:eval}
\setlength{\tabcolsep}{4pt}
\resizebox{\textwidth}{!}{%
\begin{tabular}{@{}llccccccccc@{}}
\toprule
Scenario & Domain & Obj. & Solver & Front & HV & $\rho^\star$ & $\rho$ range & Faithf.\ err & $R^2$ & Time (s)\\
\midrule
Oncology email & Pharma commercial & 3 & MO-LP (exact)   & 21  & 0.547 & 0.95 & 0.37--0.95 & 8.4\% & 0.82 & 14.9\\
CRC screening  & Public health     & 3 & MO-LP (exact)   & 135 & 0.660 & 1.00 & 0.56--1.00 & 2.9\% & 0.86 & 13.5\\
Heat-pump      & Energy            & 3 & MO-EA (NSGA-III) & 38  & 0.708 & 1.00 & 0.47--1.00 & 8.9\% & 0.73 & 13.5\\
\bottomrule
\end{tabular}}
\end{table*}

\section{Decision Traceability: Attribution and Shadow Prices}
\label{sec:trace}
A robustness score says how far a recommendation can be trusted; two further
questions decide whether it can be \emph{acted on}. ``Which segments are in this
action, and which should we prioritize?'' and ``Which constraint binds, and what is
relaxing it worth?'' We answer both for \emph{any} candidate---in particular the
robust compromise---so the object remains a trade-off choice, not a claimed single
optimum.

\textbf{Segment attribution.} Given the recommended lever setting $x^\star$, we fit
the primary outcome on the levers with segment fixed effects (a linear-probability
model $y\sim \beta^\top x + \sum_s \gamma_s\mathbf 1[\text{seg}=s]$), predict each
segment's response $p_s(x^\star)$, and weight by segment share $w_s$ to get its
share of the total engaged response, $\text{contrib}_s \propto w_s\,p_s(x^\star)$,
plus a priority rank by responsiveness. This exposes a distinction a narrative
hides: the most \emph{responsive} segment is often not the largest
\emph{contributor}. In the oncology case the recommended action's biggest
contributor is the mid-size, mid-engagement segment ($35\%$ of HCPs, $58\%$ of the
engaged response), while the smallest, most digital segment is the top \emph{priority}
by responsiveness but only $20\%$ of volume (Fig.~\ref{fig:trace}, left)---exactly
the ``which HCPs are in the $40\%$?'' question made answerable.

\textbf{Shadow prices and binding constraints.} For a candidate $x^\star$ we report
each constraint's slack (binding when $\approx 0$) and a shadow price: the marginal
change in the best achievable primary outcome per unit of relaxation,
\begin{equation}
\begin{aligned}
\pi_j&=\frac{\partial}{\partial b_j}\,V(b),\\
V(b)&=\max_{x}\,\text{primary}(x)\ \ \text{s.t.}\ \ g(x;b)\le 0,\ x\in[\ell,h],
\end{aligned}
\label{eq:shadow}
\end{equation}
estimated by perturbing $b_j$ and re-solving $V$. Where the stage is affine this is
the exact LP dual; otherwise it is a local empirical sensitivity (flagged as such).
A revealing pattern recurs: at the \emph{robust compromise} the budget is \emph{not}
binding---the compromise deliberately leaves spend on the table to stay robust---yet
budget \emph{does} bind at the reach-maximizing action, where its shadow price
($+0.004$ to $+0.007$ primary per cost unit, Table~\ref{tab:trace}) quantifies what
chasing maximum reach would cost in robustness. The energy case is sharper still:
the optimizer drives rebate depth toward zero and spends on installer coverage and
awareness instead---the model has learned the additionality lesson (salience and
targeting beat raw subsidy), and attribution shows the spend concentrating on the
most responsive, lower-income segment.

\begin{table}[t]
\centering\small
\caption{Decision traceability for the robust compromise (verified run).
``Priority'' is the most responsive segment; ``contributor'' is the largest share
of the engaged response. Budget is slack at the compromise but binds at maximum
reach, where the shadow price $\pi$ is the marginal primary gain per cost unit.}
\label{tab:trace}
\setlength{\tabcolsep}{5pt}
\resizebox{\columnwidth}{!}{%
\begin{tabular}{@{}lccc@{}}
\toprule
Scenario & Priority / contributor & Budget & Shadow $\pi$\\
\midrule
Oncology email & seg (resp) / seg ($58\%$) & slack & $+0.0036$\\
CRC screening  & seg / seg ($41\%$)         & slack & $+0.0074$\\
Heat-pump      & seg (LMI) / seg ($46\%$)   & slack & $+0.0017$\\
\bottomrule
\end{tabular}}
\end{table}

\begin{figure*}[t]
\centering
\begin{subfigure}{0.32\textwidth}\includegraphics[width=\textwidth]{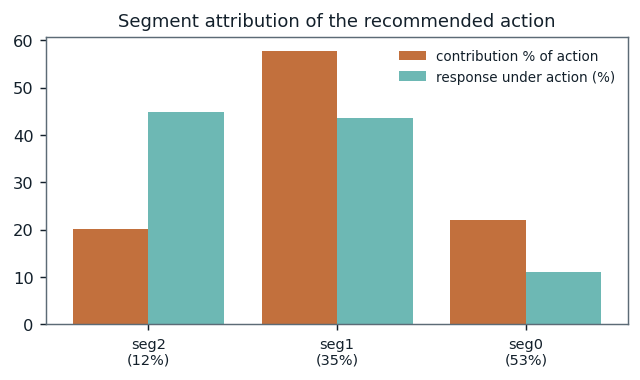}\caption{oncology email}\end{subfigure}\hfill
\begin{subfigure}{0.32\textwidth}\includegraphics[width=\textwidth]{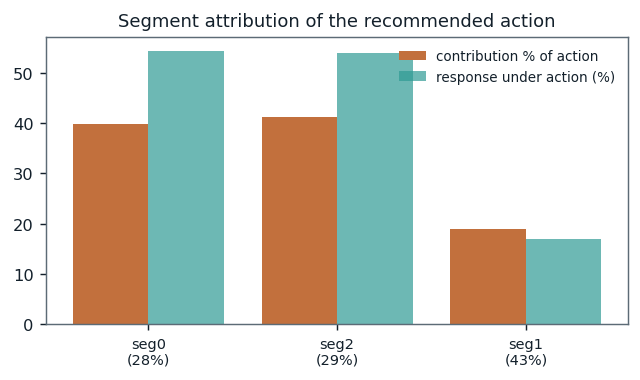}\caption{CRC screening}\end{subfigure}\hfill
\begin{subfigure}{0.32\textwidth}\includegraphics[width=\textwidth]{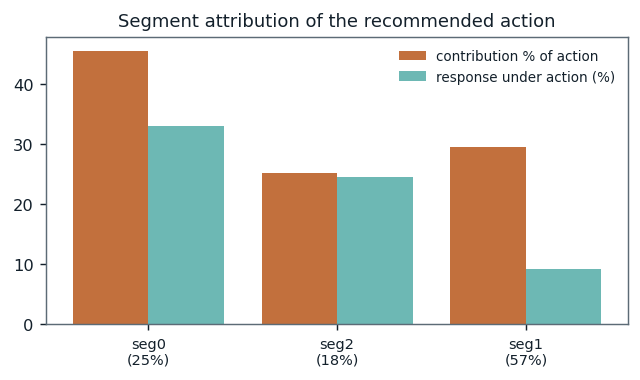}\caption{heat-pump}\end{subfigure}
\caption{Segment attribution of the recommended robust action: each segment's
contribution to the engaged response (copper) against its response rate under the
action (teal), with segment share in parentheses. The biggest contributor and the
most responsive segment frequently differ---the decision-traceability detail a
narrative answer omits.}
\label{fig:trace}
\end{figure*}

\section{External Validation under Misspecification}
\label{sec:validation}
The evaluation in \S\ref{sec:evaluation} is a system check, not a test of the
central claim---that carrying assumption uncertainty yields decisions that survive
being wrong. We test that claim out-of-sample, non-circularly, against a held-out
ground truth.

\textbf{Setup.} We model the real epistemic situation. Each lever's \emph{marginal}
effect is known from a \emph{separate} randomized trial---mailed opt-out FIT
$\approx{+}0.15$, SMS reminder $\approx{+}0.08$, navigation $\approx{+}0.10$
absolute on screening completion, anchored to published RCTs~\cite{mtics,mehta,jgim}---but
\emph{no} trial measured the levers in combination. The optimizer's surrogate is
therefore \emph{additive} in the known marginals, with no interaction term. The
\emph{true} world, revealed only at evaluation, is structurally richer: negative
interactions and saturation (diminishing returns to stacking channels) that the
additive surrogate cannot represent. The optimizer never sees the true structure,
so the test is non-circular---fit on additive marginals, evaluate on the richer
world.

\textbf{Policies.} (a) \emph{naive point}: the single-objective argmax of the
additive surrogate---the confident point plan a one-shot LLM emits;
(b) \emph{nominal}: a multi-objective compromise on the surrogate, no robustness;
(c) \emph{ARPF robust}: the robust compromise over an $\varepsilon$-contamination
$\Theta$ ($\varepsilon{=}0.5$) that hypothesizes diminishing-returns interactions,
selecting the action with the best downside (worst-$25\%$) value.

\textbf{Result.} On the held-out true world over $800$ realizations
(Table~\ref{tab:validation}), the naive plan stacks every lever and is punished by
the unmodelled saturation; the robust compromise uses the single strongest real
lever and refuses to stack. ARPF cuts mean regret by $98\%$ and tail regret
(CVaR$_{10\%}$) by $92\%$ versus the naive plan, and beats the non-robust nominal
compromise as well (Fig.~\ref{fig:validation}a). The Theorem~\ref{thm:arpf}
certificate holds (Fig.~\ref{fig:validation}b). The \emph{distributional hand-off}
matters too: handing off the naive plan's stage-1 signal as a point reports the
downstream floor satisfied ($0.58\ge0.55$), while handing off its distribution
prices a $60\%$ probability of falling below that floor (10th percentile $0.21$)---
the risk a point commitment hides.

\begin{table}[t]\centering\small
\caption{Held-out regret on the true world ($800$ realizations). The naive additive
plan stacks all levers; ARPF's robust compromise uses the strongest single lever.
Lower is better.}
\label{tab:validation}
\setlength{\tabcolsep}{5pt}
\resizebox{\columnwidth}{!}{%
\begin{tabular}{@{}lcccc@{}}
\toprule
Policy & Action & Mean & CVaR$_{10\%}$ & vs naive\\
\midrule
naive point & $(1,1,1)$ & $1.091$ & $1.471$ & ---\\
nominal & $(1,0.88,0)$ & $0.407$ & $0.627$ & $-63\%$\\
ARPF robust & $(0,1,0)$ & $\mathbf{0.021}$ & $\mathbf{0.114}$ & $\mathbf{-98\%}$\\
\bottomrule
\end{tabular}}
\end{table}

\begin{figure}[t]\centering
\begin{subfigure}{0.49\columnwidth}\includegraphics[width=\textwidth]{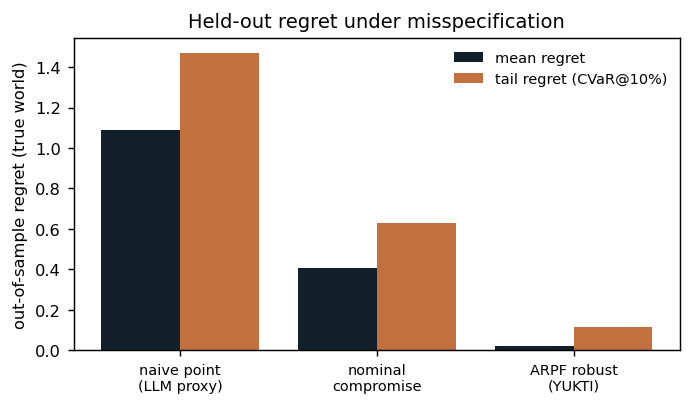}\caption{}\end{subfigure}\hfill
\begin{subfigure}{0.49\columnwidth}\includegraphics[width=\textwidth]{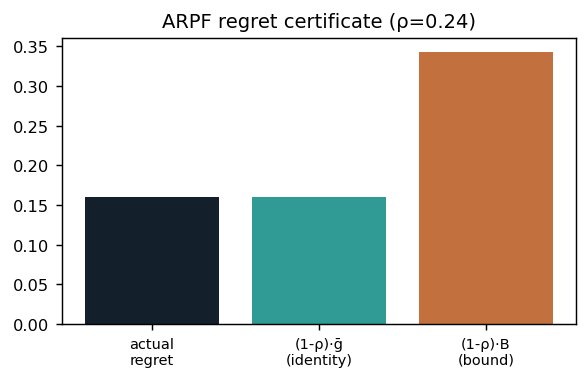}\caption{}\end{subfigure}
\caption{(a) Out-of-sample regret under structural misspecification: the ARPF
robust compromise collapses both mean and tail regret versus the naive point plan.
(b) Theorem~\ref{thm:arpf} certificate: actual pool-regret equals the identity
$(1-\rho)\bar g$ and lies under the bound $(1-\rho)B$.}
\label{fig:validation}
\end{figure}

\textbf{Scope (honest).} The true world is synthetic, but its marginal effects
equal measured RCT values and its structure is hidden from the optimizer. The
magnitude of the gap depends on interaction strength; the \emph{direction}---
robustness reduces out-of-sample and especially tail regret under
misspecification---is the claim, and it holds. Real field data would replace the
synthetic world wholesale; the methodology (fit a misspecified surrogate, evaluate
held-out, report regret) transfers unchanged.

\section{Validation on Real Data}
\label{sec:realdata}
The validations so far use worlds we generated. The sharpest test is data we did
not: the UCI Bank Marketing dataset~\cite{moro}---$41{,}188$ real telemarketing
contacts from a Portuguese bank, each a logged channel/frequency decision with a
realized conversion outcome. The decision mirrors the commercial brief in miniature:
per client segment (job $\times$ age band), choose a \emph{channel}
(cellular/telephone) and a \emph{contact-frequency} band to maximize conversion.

\textbf{Protocol.} We estimate each (segment, action) conversion on a training split
and evaluate three rules on a held-out split (mean of $12$ random $60/40$ splits):
\emph{logged} (the realized status-quo mix, no re-assignment); \emph{naive} (per
segment, argmax of the point estimate); and \emph{robust} (per segment, argmax of
the downside CVaR$_{25\%}$ of the Beta posterior---the one-dimensional instance of
YUKTI's assumption-robust compromise). An in-hindsight \emph{oracle} bounds the
achievable. Naive selection is exactly the \emph{optimizer's curse}
setting~\cite{smithwinkler}: an argmax over noisy estimates favours options that look
best by chance and regress out-of-sample.

\textbf{Result.} On held-out real data the decision layer beats the status quo
decisively---robust conversion $0.152$ vs logged $0.114$, \textbf{$+34\%$}---and
robustness also beats the naive point rule by \textbf{$+4\%$} ($0.152$ vs $0.146$),
winning in $97\%$ of segments and shrinking the gap to the oracle ($0.022$ vs
$0.028$, Table~\ref{tab:realdata}, Fig.~\ref{fig:realdata}). The mechanism is visible
in the \emph{post-decision surprise}---how much a rule's training estimate overstates
its realized value: naive $+2.6$pp versus robust $+2.0$pp. The naive plan believes
its own luck; the robust plan discounts it.

\begin{table}[t]\centering\small
\caption{Out-of-sample backtest on UCI Bank Marketing ($41{,}188$ real decisions;
mean of $12$ random $60/40$ holdouts, $25$ segments). Held-out conversion; higher is
better.}
\label{tab:realdata}
\setlength{\tabcolsep}{5pt}
\resizebox{\columnwidth}{!}{%
\begin{tabular}{@{}lccc@{}}
\toprule
Rule & Held-out conv. & vs logged & Surprise (pp)\\
\midrule
logged (status quo) & $0.114$ & --- & ---\\
naive point & $0.146$ & $+29\%$ & $+2.6$\\
ARPF robust & $\mathbf{0.152}$ & $\mathbf{+34\%}$ & $+2.0$\\
oracle (hindsight) & $0.174$ & $+53\%$ & ---\\
\bottomrule
\end{tabular}}
\end{table}

\begin{figure}[t]\centering
\includegraphics[width=0.92\columnwidth]{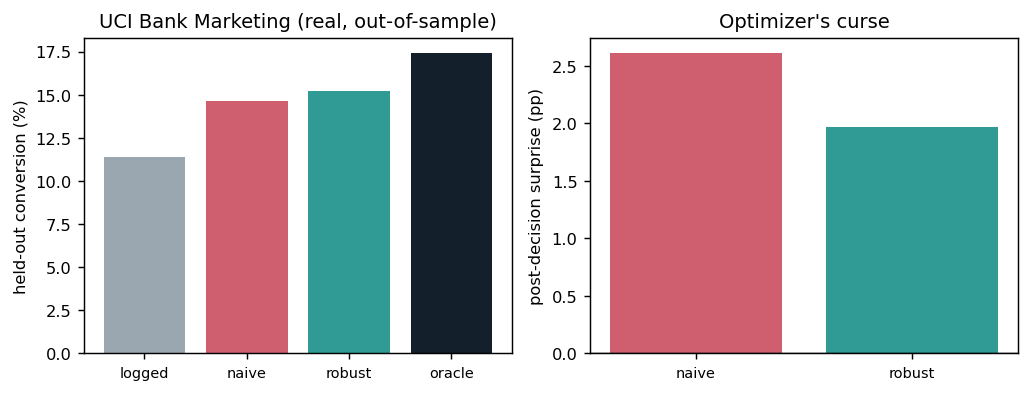}
\caption{Real-data, out-of-sample (UCI Bank Marketing). Left: held-out conversion---
the robust rule beats both the logged status quo and the naive point rule. Right: the
optimizer's curse---naive selection overstates its own pick more than robust does.}
\label{fig:realdata}
\end{figure}

\textbf{Non-stationarity (honest stress).} Under a strict \emph{out-of-time} split
(the rows are date-ordered), conversion roughly triples across periods; absolute cell
rates no longer transfer and all optimized gains compress (robust $0.214$ vs logged
$0.247$). This is the known failure mode of backtests on a drifting world, and we
state it plainly: the random-holdout result isolates the estimation-uncertainty
effect the theory addresses, while closing the non-stationarity gap needs a
recency-weighted or recourse formulation (\S\ref{sec:limitations}). Evaluation is
direct-method (held-out cell rates); a production study adds a test--control design
(Appendix~\ref{app:backtest}). Even so, this is the first result here on data the
author did not generate, and it points the same way as the theory: pricing
uncertainty beats ignoring it.

\section{Deploying in a Regulated Commercial Setting}
\label{sec:pharma}
A held-out regret win is necessary but not sufficient for a high-stakes commercial
decision. Taking an oncology channel-rebalance brief (Drug~A and Drug~B, during
loss of exclusivity) to a country team surfaces four requirements a generic
optimizer ignores; YUKTI's typed IR accommodates each as a formulation choice, and
we backtest the result against the policy that was actually run.

\textbf{A lawful, admissible action space.} Email cannot reach an oncologist who has
not opted in: under GDPR only a consented fraction $\kappa$ is addressable
($\kappa{\approx}0.45$ here), so email effect is reach-capped, not free. A frequency
cap bounds touches; a finite inventory of MLR-approved assets bounds content
($\text{depth}+\text{defense}\le A$); and a \emph{KAM override} keeps field contact
on top-decile KOLs above a floor (human-in-the-loop). These enter as constraints, so
optimization happens only inside the admissible set (Appendix~\ref{app:mlr}). This is
not cosmetic: the unconstrained ``optimal'' plan---max email, max content---is
\emph{inadmissible} (it breaches the asset and frequency caps), and only $312$ of
$1296$ grid actions are lawful. An optimum outside the lawful set is a liability, not
an answer.

\textbf{A prescribing-proxy objective, not opens.} The primary outcome is an
NBRx-uplift proxy in which the field relationship dominates for high-value
oncologists and email is a consent-capped, fatigue-prone, \emph{mediated}
contributor; opens/CTR feed an engagement-quality term but are never the objective.
Message is not a scalar ``content'' dial: Drug~A rewards scientific-depth content
(a targeted line-of-therapy), Drug~B rewards LoE-defense content, so the recommended
content architecture differs by brand by construction.

\textbf{Out-of-time backtest against the logged policy.} We fit the surrogate on past
quarters, freeze it, and evaluate each policy's action on held-out future quarters of
a true world it never saw, comparing realized NBRx against the status quo that was
actually run (Appendix~\ref{app:backtest}). Table~\ref{tab:pharma} reports the result.
The robust compromise keeps field (KAM preserved), \emph{declines} the aggressive
email shift the brief proposed, optimizes content to each brand, and---crucially---
neutralizes the downside (5\% Value-at-Risk $\approx 0$) where the naive plan carries
a large negative tail. The \emph{nominal} compromise that trims expensive field
\emph{underperforms the status quo} (Fig.~\ref{fig:pharma}): naively rebalancing away
from the relationship loses prescriptions to channel substitution---the failure the
field organization predicts.

\begin{table}[t]\centering\small
\caption{Out-of-time backtest vs the logged status quo (Drug~A): NBRx uplift
\emph{relative to logged} and 5\% revenue Value-at-Risk in \eur M (illustrative
oncology unit economics, Appendix~\ref{app:var}). The unconstrained optimum is
inadmissible (breaches MLR/frequency caps).}
\label{tab:pharma}
\setlength{\tabcolsep}{5pt}
\resizebox{\columnwidth}{!}{%
\begin{tabular}{@{}lccc@{}}
\toprule
Policy & Action $(f,e,d,v)$ & NBRx & rev-VaR (\eur M)\\
\midrule
logged (status quo) & $(.85,.35,.45,.40)$ & --- & $-10.0$\\
naive (unconstr.) & $(1,1,1,0)$ & \emph{inadmiss.} & ---\\
naive (admissible) & $(1,.6,1,0)$ & $+15\%$ & $+1.7$\\
nominal & $(.6,.6,1,0)$ & $-8\%$ & $-17.2$\\
ARPF robust & $(1,0,.8,0)$ & $\mathbf{+12\%}$ & $-0.3$\\
\bottomrule
\end{tabular}}
\end{table}

\begin{figure}[t]\centering
\includegraphics[width=0.9\columnwidth]{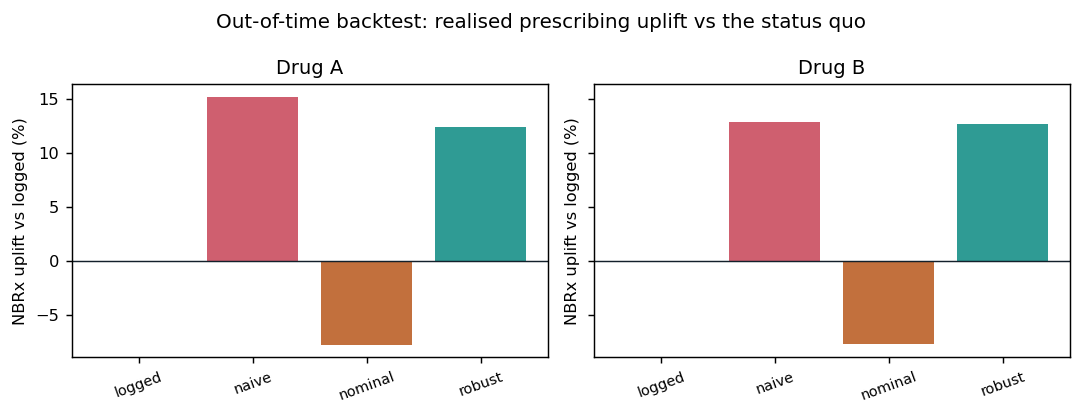}
\caption{Out-of-time realized NBRx uplift vs the logged status quo, by brand. The
robust compromise gains modestly and protects the downside; the nominal field-trim
loses to channel substitution; the recommended content differs by brand (Drug~A
depth, Drug~B defense).}
\label{fig:pharma}
\end{figure}

\textbf{Decision economics (euros, not $\rho$).} The NBRx-uplift distribution maps to
revenue and share-of-voice with a downside band: the robust plan's expected
incremental revenue is positive with 5\% VaR near zero, whereas the field-trimming
nominal plan risks an eight-figure shortfall (Table~\ref{tab:pharma}). A country GM
transacts in euros-at-risk and accountability, not value units;
Appendix~\ref{app:var} gives the mapping.

\textbf{Deployment posture.} The honest role is a decision \emph{stress-testing and
triangulation} layer, not a system of record: it pressure-tests a human plan, prices
its downside, and flags the lawful action space---then writes its recommendation and
decision-traceability into the validated commercial stack (e.g.\ as a prior or
guardrail for a Veeva next-best-action engine), where a rep sees it in-flow and the
KAM override is always available. The synthetic data here is benchmark- and
RCT-anchored; a real Veeva/IQVIA backtest is the one step that converts this from a
credible simulation into evidence, and the protocol is built for exactly that swap.

\ypart{IV}{Boundaries and Position}
\section{LLM Reasoning, Function Optimization, or YUKTI?}
\label{sec:comparison}
A fair objection: could a capable reasoning model (an LRM) just \emph{read} the brief
and answer, or does classical single-objective optimization already suffice? We test
both against YUKTI on the same held-out harness, separating \emph{what each method is
for} from \emph{what it can guarantee} (Table~\ref{tab:capability}).

\textbf{Capability.} The three occupy different rungs. An LRM excels at the
\emph{front end}---turning a worded situation into structured quantities---which is
precisely why YUKTI uses one there. Asked instead to \emph{be} the optimizer, an LRM
has no mechanism to guarantee feasibility, to enumerate a Pareto frontier without an
arbitrary scalarization, to propagate coefficient uncertainty, or to certify regret;
empirically, LLM-as-optimizer degrades sharply with problem size and was never meant
to beat specialized solvers~\cite{opro}, and LLM-emitted solutions are frequently
\emph{infeasible} in constrained tasks because the model describes a solution without
tracking the constraints~\cite{llmco}. Classical single-objective optimization
guarantees feasibility and a nominal optimum, but collapses conflicting objectives
into one weighted scalar and returns an uncertainty-free \emph{point}. YUKTI keeps the
LRM for intake and replaces the solve-and-hedge with machinery that has these
properties.

\textbf{Outcome.} On the misspecification problem (\S\ref{sec:validation}), evaluated
identically: an LRM handed the \emph{correct} marginal effects and reasoning
additively commits to the stacked plan (mean held-out regret $1.03$; because decoding
is stochastic, a run-to-run spread of $0.12$ with no stable optimum); classical
single-objective optimization lands at the same fragile corner ($1.07$); YUKTI's
robust compromise scores $0.022$---about $47\times$ lower. The lesson is not that the
LRM reasons badly (it is handed the right numbers) but that a \emph{point} plan,
however reasoned, ignores the assumption uncertainty and the objective conflict that
dominate the decision. The same ordering holds out-of-sample on real data
(\S\ref{sec:realdata}: the naive point rule is the LRM/classical proxy; robust beat it
by $4\%$ with a smaller optimizer's-curse surprise), and in the regulated setting
(\S\ref{sec:pharma}) the confident LRM plan---``shift hard to digital, maximize
content''---is not merely fragile but \emph{inadmissible}: it breaches the MLR and
consent caps, the infeasibility mode~\cite{llmco} made concrete in euros.

\begin{table*}[t]\centering\small
\caption{Three ways to turn a worded decision into a plan. An LRM is the best
\emph{formulator} but cannot guarantee what a high-stakes decision needs; classical
single-objective optimization guarantees feasibility and a nominal optimum but returns
an uncertainty-free point; YUKTI uses the LRM for intake and supplies the rest. Lower
block: identical held-out evaluation (\S\ref{sec:validation}--\ref{sec:pharma}).
\checkmark~= provides; $\sim$~= partial; $\times$~= absent.}
\label{tab:capability}
\setlength{\tabcolsep}{6pt}
\resizebox{\textwidth}{!}{%
\begin{tabular}{@{}lccc@{}}
\toprule
& LRM (reasoning) & Single-obj.\ optimization & \textbf{YUKTI}\\
\midrule
NL situation $\rightarrow$ structured model & \checkmark\ (strong) & $\times$\ (manual) & \checkmark\ (LRM front-end)\\
Feasibility guarantee & $\times$~\cite{llmco} & \checkmark & \checkmark\\
Exact / numerical solve & $\sim$\ (small; degrades~\cite{opro}) & \checkmark & \checkmark\ (routed)\\
Pareto frontier, no forced scalarization & $\times$ & $\times$ & \checkmark\\
Coefficient-uncertainty propagation & $\times$ & $\times$ & \checkmark\ (typed IR)\\
Assumption-robustness / regret certificate & $\times$ & $\times$ & \checkmark\ (Thm~\ref{thm:arpf})\\
Decision traceability (attribution, duals) & $\sim$\ (narrative) & $\sim$\ (duals) & \checkmark\\
Deterministic / reproducible & $\times$ & \checkmark & \checkmark\\
\midrule
Held-out regret, misspecification (\S\ref{sec:validation}) & $1.03$ & $1.07$ & $\mathbf{0.022}$\\
Run-to-run instability & $0.12$ & $0$ & $0$\\
Lawful in regulated setting (\S\ref{sec:pharma}) & $\times$\ (inadmissible) & \checkmark & \checkmark\\
\bottomrule
\end{tabular}}
\end{table*}

\textbf{Perspective.} The boundary is clean. Use an LRM for what it is---a superb
\emph{formulator} of messy situations into structure. Do not ask it to be the
\emph{solver}: feasibility, frontiers, uncertainty, and a regret certificate are not
properties a next-token predictor provides, and ``reasoning'' tokens improve the
\emph{formulation}, not the optimization. That division of labour---LRM in front,
verifiable optimization behind---is YUKTI, and it is the constructive form of the
mimicry-of-computation thesis.

\section{Long-Range Causal Dependence: Where the Hand-off Must Become a Policy}
\label{sec:causal}
YUKTI solves stages in topological order and passes a signal forward
(\S\ref{sec:multistage}). That is adequate when an early action only changes the
\emph{state} the next stage sees; it is \emph{not} adequate when an early action
causally reshapes a later stage's \emph{mechanism}---a long-range dependency
$a_1\!\rightarrow\!\text{mechanism}_2$, not $a_1\!\rightarrow\!\text{state}_2$. Such
dependence is the rule in the settings these methods target: an early therapy
changes how a later one acts; an early brand investment changes the elasticity of a
later channel. The field is converging here---commercial causal-decision platforms
pair an LLM front-end with causal graphs and intervention/recourse search (causaLens
\emph{decisionOS}; Alembic's real-time causal recomputation of saturation and lift),
and the long-horizon LLM-agent literature is consumed by \emph{credit
assignment}~\cite{hcapo}. The rigorous formalism is the dynamic-treatment-regime
(DTR) line, and it warns about exactly our hand-off: when late-stage outcome models
are misspecified and plugged into early-stage optimization, \emph{bias accumulates
from stage to stage and grows with the horizon}~\cite{huangning}, so the correct
object is a \emph{global / backward-induction} policy, not a forward myopic
pass~\cite{murphy,robins}.

We make this testable. On a controlled causal chain where an early action $a_1$
modulates stage-2 effectiveness with coupling strength $\kappa$ (so
$a_1\!\rightarrow\!\text{mechanism}_2$), we compare YUKTI's myopic hand-off against a
backward-induction \emph{causal-trajectory} policy over the same robust (ARPF/CVaR)
objective (Fig.~\ref{fig:causal}).

\textbf{(C1) Myopia penalty.} At $\kappa{=}0$ the two coincide---YUKTI's present
hand-off is the no-coupling special case. As coupling grows, the myopic pass (which
optimizes stage~1 on its own value and under-invests in the early action) incurs
trajectory regret growing roughly linearly in $\kappa$, exceeding $1.0$ value unit at
strong coupling, while the causal-trajectory policy holds regret near zero. The DTR
bias-accumulation result, reproduced.

\textbf{(C2) The certificate does not compose.} Theorem~\ref{thm:arpf} certifies
\emph{per-decision} pool-regret. But the stage-wise product $\rho_1\rho_2$ that a
forward pass reports \emph{overstates} the myopic plan's true trajectory robustness
$\rho_{\mathrm{traj}}$, and the optimism gap grows with $\kappa$: at strong coupling
the per-stage certificate rates the myopic plan ${\sim}0.96$ robust while its
probability of being trajectory-optimal is ${\approx}0$. The certificate is sound per
decision and \emph{unsound} across a causally-coupled trajectory.

\begin{figure}[t]\centering
\includegraphics[width=\columnwidth]{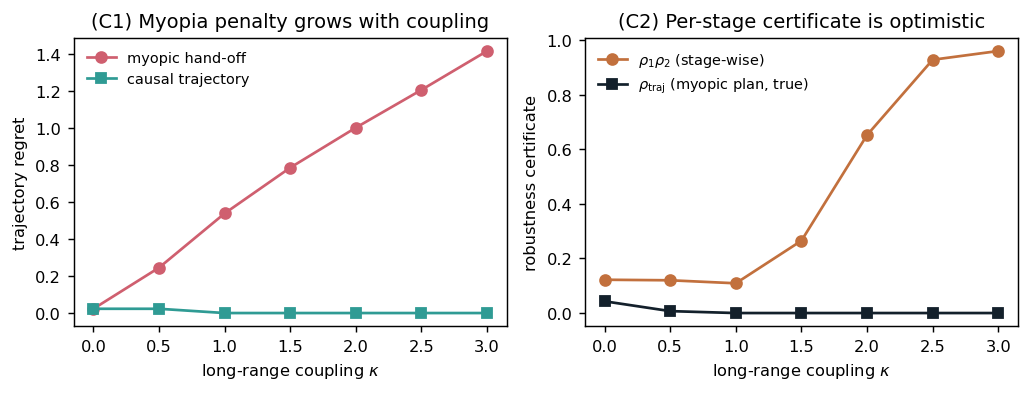}
\caption{Long-range causal coupling $\kappa$ ($a_1\!\rightarrow\!\text{mechanism}_2$).
Left (C1): the myopic hand-off's trajectory regret grows with coupling while the
backward-induction causal-trajectory policy stays near zero; at $\kappa{=}0$ they
coincide. Right (C2): the stage-wise certificate $\rho_1\rho_2$ stays high while the
myopic plan's true trajectory robustness $\rho_{\mathrm{traj}}$ collapses---the
per-stage certificate is increasingly optimistic.}
\label{fig:causal}
\end{figure}

\textbf{What must change.} Three modifications follow, each a concrete extension of an
existing component. (i) \emph{Typed-causal IR}: the Typed Proposition graph gains
temporal causal edges typed by lag and mechanism (does $a_1$ shift a later state or a
later coefficient?), with time-varying confounding represented rather than assumed
away~\cite{robins}. (ii) \emph{Backward-induction robust policy}: where coupling is
present the forward hand-off is replaced by a backward-induction solve over the
trajectory---robust $Q$-/structural-nested style~\cite{robustdtr}---with ARPF
resampling the causal \emph{model} (structure and coupling), not only per-stage
coefficients. (iii) \emph{Trajectory certificate and causal credit}: $\rho$ and the
regret bound move to the trajectory, and decision-traceability (\S\ref{sec:trace})
generalizes to \emph{causal credit assignment}---path-specific contributions of early
actions to late outcomes---precisely the problem the long-horizon agent literature is
attacking~\cite{hcapo}.

\textbf{Scope.} We diagnose and bound the failure and verify the direction; we do not
ship a causal-DP solver here. The $\kappa{=}0$ result confirms the present system is
correct in the regime it was built for; the $\kappa{>}0$ results say what must change,
and that change is well-trodden in the DTR and causal-RL literature~\cite{causalrl}---
an identification-and-engineering problem, not a conceptual gap.

\section{Positioning and Novelty}
\label{sec:positioning}
We separate what we reuse from what we claim.
\textbf{Established (reused):} the five-element formulation discipline and
IR-plus-solver loop~\cite{llmopt,optimus,orpilot}; augmented
$\varepsilon$-constraint / AUGMECON2~\cite{augmecon2}; NSGA-II/III, MOEA/D,
Das--Dennis directions, and pymoo's MCDM~\cite{nsga2,nsga3,moead,dasdennis,pymoo};
HiGHS and SciPy~\cite{highs,scipy}; the notion of integrating over frontier
uncertainty from MOBO~\cite{qnehvi}.
\textbf{Contributed (new, to our knowledge):}
(1) a \emph{Typed Proposition IR} that carries shape priors, a coefficient
uncertainty distribution, and provenance---versus the point-valued IRs of prior
autoformulation systems;
(2) \emph{ARPF}, propagating that IR-level assumption uncertainty onto the
frontier to produce per-decision robustness scores and robust compromise
selection inside an autoformulation pipeline---the model-assumption analogue of
qNEHVI's observation-noise treatment;
(3) a \emph{heterogeneous multi-stage hand-off} that couples exact and
evolutionary stages across a decision DAG; and
(4) a \emph{numeric structure-aware router} that picks exact / NLP / EA per stage
by probing affinity and convexity on the relaxation;
(5) \emph{SRJANA}, benchmark-anchored contextual synthesis that manufactures a
faithful data foundation when none exists (the synthesized analytics layer); and
(6) \emph{decision-traceability} reporting---segment attribution and
shadow-price/binding-constraint analysis---that makes a recommended action
auditable.
The synthesis as a whole---worded situation to multi-stage, multi-objective,
assumption-robust recommendation---is, to our knowledge, not available in one
system. We emphasize again: \emph{no benchmark-SOTA claim is made}; evaluation
here is a verified demonstration, not a comparison.

\textbf{Relation to decision-support systems.} The three-layer reading of our
pipeline---analytics, decision/model, and a presentation layer---is deliberately
classical: it mirrors the data/model/dialog decomposition of decision-support
systems~\cite{sprague}. We claim no novelty in the \emph{layering}; the
contribution is the \emph{mechanisms} that make the decision layer trustworthy
under LLM-objectified inputs (typed-uncertain IR, ARPF, the router, the hand-off,
SRJANA, and decision-traceability reporting). Framed against that lineage, the
problem an LLM introduces is \emph{mimicry of computation}: fluent output with the
form of an optimized recommendation but without a solved model, stated
coefficients, sensitivity, or robustness. Source traceability (RAG provenance) does
not address this; decision traceability does. YUKTI's value is therefore not that
an LLM ``cannot'' do the analysis---a capable model can---but that the
recommendation it emits arrives with an explicit, re-runnable, assumption-aware
artifact rather than as unverifiable prose.

\section{Limitations}
\label{sec:limitations}
\textbf{Illustrative coefficients.} The demonstration's numbers are elicited
placeholders, not measured field data; the contribution is structural. The
external validation (\S\ref{sec:validation}) anchors a held-out world to measured
RCT effects and shows the robust compromise reduces out-of-sample regret, but real
proprietary data would replace the synthetic world entirely; ARPF exists to
quantify the exposure, and the tornados say which numbers to measure first.
\textbf{ARPF re-ranks a fixed pool.} Equation~\ref{eq:arpf} resamples assumptions
and re-ranks the \emph{discovered} solutions rather than re-searching from
scratch each draw. This is exactly why Theorem~\ref{thm:arpf} certifies
\emph{pool}-regret (against the best-in-pool) rather than oracle regret; it is
efficient and adequate when the nominal pool is dense near the frontier, but can
understate robustness of decisions far from the nominal optimum. Periodic full
re-searches are a configurable, costlier option.
\textbf{Hand-off, not recourse.} The multi-stage layer commits a compromise and
passes its signals forward; it does not optimize a recourse policy over a scenario
tree as in multistage stochastic programming~\cite{birgebook,shapiro}. This is
appropriate for hierarchical strategy decisions but is not a substitute for
recourse when later stages can genuinely revise earlier ones.
\textbf{Router heuristics.} Affinity/convexity are inferred from finite samples
with tolerances; pathological functions can be misread. Misrouting is
conservative (EA solves a superset of what exact handles) but may forgo exactness.
\textbf{LLM front-end.} This paper specifies and demonstrates the engine; wiring a
production LLM elicitor with self-repair~\cite{optirepair} on top of the typed IR
is engineering we describe but do not benchmark here.
\textbf{Real data and non-stationarity.} The real-data backtest (\S\ref{sec:realdata})
is genuine out-of-sample evidence on data we did not generate, but under a strict
out-of-time split a tripling of the base rate compresses the gains: a drifting world
defeats any frozen response surface, ours included. The honest scope is that pricing
\emph{estimation} uncertainty helps out-of-sample; pricing \emph{distribution shift}
needs a recency-weighted or recourse extension. Evaluation there is direct-method,
not a randomized test--control study.

\section{Future Work}
Natural extensions: (i) an expensive-evaluation mode using qNEHVI when each
proposition evaluation is a simulation rather than a formula~\cite{qnehvi}; (ii)
true multistage recourse with a scenario tree built from the proposition
distributions, recovering the stochastic-programming view~\cite{birgebook};
(iii) calibrating propositions against data so provenance can graduate from
\emph{assumed} to \emph{benchmark}; and (iv) a benchmark of formulation fidelity
and robustness against single-shot LLM baselines on multi-objective instances.

\section{Conclusion}
YUKTI keeps the LLM in the role it does well---reading a situation and naming its
structure---but changes what it produces: not a brittle point model with one
objective, but a \emph{typed, uncertain} proposition graph whose stages are routed
to the right solver, chained by Pareto hand-off, and ranked by how often each
choice survives the very assumptions that were invented to objectify the brief.
The result is a recommendation that comes with its own statement of fragility.
We provide a working implementation and an end-to-end demonstration, and we are
explicit that the novelty is the uncertainty-typed IR and the ARPF mechanism, not
the underlying solvers or any benchmark claim.

\appendix

\section{From Contextual Analysis to Formulation}
\label{app:contextual}
This appendix records the design trail: the core statements that emerged from
analysing real ``tell-me-what-to-do'' briefs, and how each was translated into a
concrete framework mechanism. The framework is not an arbitrary architecture; every
component answers an observation. Table~\ref{tab:contextual} is the compact map;
the prose below elaborates the reasoning.

\begin{table*}[t]
\centering\small
\caption{Contextual analysis $\rightarrow$ framework. Each core statement discovered
while reading real briefs becomes a specific mechanism. This is the derivation of
YUKTI's structure from first observations.}
\label{tab:contextual}
\renewcommand{\arraystretch}{1.15}
\begin{tabular}{@{}p{0.40\textwidth}p{0.27\textwidth}p{0.27\textwidth}@{}}
\toprule
\textbf{Core statement (observed)} & \textbf{Implication} & \textbf{Mechanism}\\
\midrule
``I know my levers, signals and limits, not variables, an objective and constraints.'' & The user speaks an intent vocabulary, not a solver vocabulary. & Elicitation to \emph{levers / signals / limits}; the Typed Proposition IR.\\
``Reach, cost and risk genuinely trade off.'' & A single weighted scalar hides the decision. & Pareto frontiers; no a-priori scalarization.\\
``To make it numeric, someone has to invent a coefficient.'' & Objectified numbers are assumptions, not facts. & Per-coefficient \emph{uncertainty distribution} $\Theta$ + provenance (given/assumed/benchmark).\\
``A plan that is right only if the guess is right is not actionable.'' & Optimality must be qualified by fragility. & ARPF: resample $\Theta$, score $\rho(x)$ (Eq.~\ref{eq:rho}); pick a robust compromise.\\
``First I design the program, then I decide how hard to roll it out.'' & Decisions are sequential; stage~1 fixes stage~2's playing field. & Multi-stage Pareto \emph{hand-off} across a decision DAG.\\
``Some pieces are linear, some saturate, some interact.'' & One solver cannot serve every stage. & Structure-aware router (affinity/convexity probes) $\rightarrow$ exact / NLP / EA.\\
``This is a fresh brief; there is no dataset yet.'' & Point-valued IRs have nothing to fit. & SRJANA: synthesize a contextual data foundation from spec + web anchors.\\
``The literature fixes both a level and a contrast.'' & Means alone under-determine the data. & Moment system: match marginal \emph{and} effect ratio (Eq.~\ref{eq:srjana-moments}).\\
``Synthetic data is only useful if it is trustworthy.'' & Must verify before optimizing. & Self-consistency gate: $100\%$ sign recovery; reject data that cannot reproduce its anchors.\\
``Real extracts are messy (open-rate inflation, consent gating).'' & Clean synthetic data misleads analytics. & Inject pathologies; prefer click-to-open over raw open rate.\\
``The laggards aren't sending less---they're targeting worse.'' & Gaps are capability gaps, not effort gaps. & Cross-group decomposition on the best-practice \emph{lever mix}.\\
``Much of the subsidy funds people who'd have acted anyway.'' & Headline uptake $\ne$ impact. & Separate \emph{adoption} from \emph{additional} (counterfactual) outcomes.\\
``Salience and targeting can beat raw spend.'' & The cheapest lever is not the weakest. & Driver analysis (importance + uplift) feeding the cost-aware optimizer.\\
``I have to defend this to a committee.'' & Recommendations must be auditable. & Provenance tags, tornado sensitivity, and a per-solution $\rho$.\\
\bottomrule
\end{tabular}
\end{table*}

\textbf{A.1 From words to a typed program.} Briefs name what is \emph{controlled}
(levers), what is \emph{cared about} (signals), and what must be \emph{respected}
(limits). We therefore do not ask the LLM for variables and an objective; we ask it
to populate this intent vocabulary, and we make each numeric relationship a
\emph{typed proposition} $\varphi(x,u;\theta)$ rather than a constant. The type
carries a shape prior (e.g.\ ``increasing, saturating''), a coefficient
distribution, and a provenance tag, so the very first artifact already encodes what
is known versus guessed.

\textbf{A.2 Conflict and assumptions are first-class.} Two observations dominate
real briefs: objectives conflict, and the numbers that make them comparable are
invented. The first forbids scalarization (we keep a frontier); the second forbids
treating any coefficient as ground truth (we attach $\Theta$ and, through ARPF,
report how often each candidate stays feasible and non-dominated under resampled
assumptions). Robustness is thus not a post-hoc check but the selection criterion.

\textbf{A.3 Sequence.} ``Design, then roll out'' is two coupled problems: the
roll-out's achievable share-of-voice depends on the engagement the design stage
secured. We model this as a Pareto hand-off---anchor signals from stage~1 enter
stage~2 as fixed inputs---so the second decision is optimized \emph{given} the
first, not in isolation.

\textbf{A.4 No data is the common case.} A genuinely new brief has no table to
query. SRJANA closes this gap by treating web-retrieved benchmarks as moment
constraints and calibrating a segment-mixture generator to them, then fitting the
propositions to the synthesized data so the optimizer inherits honest, data-derived
uncertainty. Because literature reports both levels and contrasts, the calibration
matches a mean \emph{and} a ratio; because synthetic data can lie, we gate on
self-consistency before optimizing.

\textbf{A.5 Structure varies.} Affine stages should be solved exactly; convex
stages by NLP; nonconvex ones by an evolutionary method. Rather than ask the user,
the router probes the stage numerically (interpolation for affinity, Jensen-gap
sampling for convexity) and dispatches accordingly---which is why the three
demonstration scenarios land on different solver classes from one code path.

\textbf{A.6 Domain pathologies and equity.} Contextual analysis repeatedly surfaced
domain-specific traps that a naive formulation would miss: Apple Mail Privacy
Protection inflating open rates (so click-to-open is the trustworthy signal);
laggard markets that send \emph{more} but target worse (so the gap analysis is on
the lever mix, not volume); and subsidies that mostly fund non-additional adopters
(so we separate adoption from \emph{additional} installs and let the optimizer
pursue impact, not headline uptake). Each became an explicit modelling choice
rather than an afterthought.

\section{Consolidated Equations}
\label{app:equations}
\textbf{Typed proposition.} Every quantitative relationship is
$\varphi(x,u;\theta)$ with levers $x$, upstream inputs $u$, coefficients
$\theta\sim\Theta$, and a shape prior $s\in\{+,-,\cap,\cup,\text{lin}\}$.

\textbf{Router.} A stage is \emph{affine} if for random $x_1,x_2,\lambda$,
$\|F(\lambda x_1+(1{-}\lambda)x_2)-\lambda F(x_1)-(1{-}\lambda)F(x_2)\|\le\epsilon$;
\emph{convex} if the Jensen gap is $\ge-\epsilon$ for all sampled triples;
otherwise \emph{nonconvex}, routing to exact / NLP / EA respectively.

\textbf{Augmented $\varepsilon$-constraint (AUGMECON2).} For objectives
$f_1,\dots,f_m$ (min-sense),
\begin{equation}
\min_{x}\; f_1(x)-\delta\sum_{k=2}^{m}\frac{s_k}{r_k}\quad
\text{s.t.}\;\; f_k(x)+s_k=e_k,\; s_k\ge0,\; x\in\mathcal X,
\end{equation}
swept over a grid of $e_k$ between each objective's payoff-table ideal and nadir,
with range $r_k$ and a small $\delta$ to exclude weakly-dominated points.

\textbf{ARPF.} Robustness $\rho(x)$ as in Eq.~\ref{eq:rho}; the robust compromise
minimizes a compromise distance to the ideal among solutions with $\rho\ge$ the
frontier median.

\textbf{Compromise programming.} With per-objective ideal $z^\star_k$ and nadir
$z^{\mathrm{nad}}_k$, choose
$\arg\min_x \big(\sum_k w_k\,|f_k(x)-z^\star_k|^p/|z^{\mathrm{nad}}_k-z^\star_k|^p\big)^{1/p}$,
$p{=}2$.

\textbf{SRJANA moment system.} For each rate signal with benchmark mean
$\mu^\star$ and effect ratio $r^\star$ on a binary feature $\ell$, solve for
intercept $\alpha$ and effect-scale $\kappa$,
\begin{equation}
\bar p(\alpha,\kappa)=\mu^\star,\qquad
\frac{\bar p_{\ell=1}(\alpha,\kappa)}{\bar p_{\ell=0}(\alpha,\kappa)}=r^\star ,
\label{eq:srjana-app}
\end{equation}
with $\bar p(\cdot)=\tfrac1{NM}\sum_{c,u}\sigma(\alpha+\kappa\beta_\ell\tilde x_{c,\ell}+L^{-}_{c,u})$
on probabilities (smooth in $(\alpha,\kappa)$), then sample Bernoulli outcomes;
score/cost signals calibrate in closed form. Identifiability holds because the two
moments are monotone in $(\alpha,\kappa)$ over the relevant range.

\section{Agent Invocation via A2A}
\label{app:a2a}
YUKTI is exposed as an \textbf{A2A} (Agent2Agent) agent so other agents can
discover and call it. It advertises an Agent Card at
\texttt{/.well-known/agent-card.json} (protocol \texttt{0.3.0}, transport JSON-RPC
2.0 over HTTP) declaring five skills: \texttt{list\_scenarios},
\texttt{optimize\_scenario}, \texttt{formulate\_and\_optimize} (consumes a
structured \emph{Brief}), \texttt{synthesize\_data} (SRJANA only), and
\texttt{decision\_trace} (segment attribution + shadow prices for the robust
action). A caller
discovers the card, then issues \texttt{message/send}; the engine runs
synchronously and returns a \texttt{Task} in state \texttt{completed} whose
\texttt{artifacts} carry a \texttt{DataPart} with the faithfulness report,
personas, drivers, cross-group gap and the robust compromise. \texttt{tasks/get}
re-fetches by id. A minimal call:

{\footnotesize\begin{verbatim}
POST /  (JSON-RPC 2.0)
{"jsonrpc":"2.0","id":"r1",
 "method":"message/send",
 "params":{"message":{"role":"user",
   "messageId":"m1",
   "parts":[{"kind":"data","data":{
     "skill":"optimize_scenario",
     "input":{"scenario":"screening"}}}]}}}
-> result: {"kind":"task","status":
   {"state":"completed"},
   "artifacts":[{"parts":[{"kind":"data",
     "data":{...robust compromise...}}]}]}
\end{verbatim}}

\noindent Natural-language parsing remains the caller's responsibility; the engine
consumes the structured brief and runs deterministically, which keeps the
LLM-as-elicitor and engine-as-solver boundary clean.

\section{Admissible Action Space (MLR / Consent / Frequency / KAM)}
\label{app:mlr}
The lawful action space restricts the levers $x=(\text{field } f,\text{email
freq } e,\text{depth } d,\text{defense } v)\in[0,1]^4$ by:
\emph{Consent (GDPR).} Email reaches only the opted-in fraction $\kappa$, so the
email contribution enters as $\kappa\,(\text{open}+4\,\text{CTR})\,e$ rather than a
free additive term ($\kappa{=}0.45$; pharma email open $\approx19\%$, CTR
$\approx2.5\%$). \emph{Frequency cap.} $e\le F$ ($F{=}0.6$). \emph{MLR inventory.}
content is bounded by the approved-asset budget, $d+v\le A$ ($A{=}1.2$).
\emph{KAM override.} field on top-decile KOLs is floored, $f\ge \underline f$
($\underline f{=}0.5$), a hard human-in-the-loop invariant. A candidate is
\emph{admissible} iff all hold; the optimizer and ARPF operate only on the
admissible subset. In the demonstration $312/1296$ grid actions are admissible and
the unconstrained surrogate optimum is not among them.

\section{Out-of-Time Backtest Protocol}
\label{app:backtest}
The backtest mimics a real before/after evaluation so it transfers unchanged to
Veeva/IQVIA data. (1) \emph{Split.} Fit the additive surrogate on history up to
$t_0$; freeze it. (2) \emph{Recommend.} Compute the four policies (naive,
admissible-naive, nominal, ARPF-robust) on the frozen surrogate within the
admissible set. (3) \emph{Hold-out evaluation.} Draw the true world (structurally
richer than the surrogate: field/email substitution, frequency fatigue, content
saturation, consent-capped reach) for the post-$t_0$ quarters and score realized
NBRx for each policy and for the \emph{logged} status quo actually run. (4)
\emph{Report.} Mean NBRx uplift vs logged and the $5\%$ lower tail. A real study
substitutes a test--control design: size each arm for a $12$-week NBRx read with the
therapy-area NBRx:TRx ratio, randomize matched HCP cohorts, and compare the robust
recommendation against the incumbent allocation. The simulation is RCT/benchmark
anchored; only the data source changes.

\section{Decision Economics: Revenue-at-Risk}
\label{app:var}
The NBRx-uplift distribution under $\Theta$ maps to money. For a target panel of $N$
high-value oncologists, addressable new patients $m$ per HCP per year, and
incremental annual value $\pi$ per started patient, expected incremental revenue is
$\widehat{\mathbb E}[\Delta\text{rev}]=\mathbb E_\Theta[\Delta\text{NBRx}]\cdot m N\pi$,
and the $5\%$ revenue Value-at-Risk is the corresponding quantile,
$\text{VaR}_{5\%}=Q_{0.05}(\Delta\text{NBRx})\cdot mN\pi$, giving the downside band a
GM can act on (illustrative $N{=}600$, $m{=}8$, $\pi{=}\in45$k). Share-of-voice is
mapped linearly from the field/reach terms. These parameters are explicit
assumptions the affiliate replaces with its own; the \emph{method}---price the whole
robustness distribution, report expected value and downside together---is the point.

\end{document}